\documentclass[10pt,  draftclsnofoot, onecolumn,article,romanappendices]{IEEEtran}

\usepackage{setspace}

\usepackage{cite}
\ifCLASSINFOpdf
   \usepackage[pdftex]{graphicx}
\else
  \usepackage[dvips]{graphicx}
\fi
\usepackage{amsmath,amssymb,amsthm,mathrsfs,amsfonts,dsfont,stmaryrd}
\usepackage{array}
\usepackage{dsfont}
\usepackage{amsbsy}
\usepackage{epstopdf}
\usepackage{graphicx,color}
\usepackage{hyperref}
\usepackage[font=footnotesize]{caption}% or \usepackage{caption}

\newtheorem{lemma}{Lemma}
\newtheorem{proposition}{Proposition}
\newtheorem{remark}{Remark}
\newtheorem{corollary}{Corollary}
\long\def\symbolfootnote[#1]#2{\begingroup%
\def\thefootnote{\fnsymbol{footnote}}\footnote[#1]{#2}\endgroup}
\newtheorem{theorem}{Theorem}
\newtheorem{definition}{Definition}

\usepackage{booktabs}

\newcommand{\dv}{\mathbf} % determenistic vector
\newcommand{\mc}{\mathcal} % determenistic vector
\newcommand{\mkv}{-\!\!\!\!\minuso\!\!\!\!-}
\newcommand{\mb}{\mathbf} % determenistic vector

 \allowdisplaybreaks

\usepackage{algorithm} 
\usepackage{algpseudocode}

\algnewcommand{\Inputs}[1]{%
  \State \textbf{Inputs:}
  \Statex \hspace*{\algorithmicindent}\parbox[t]{.8\linewidth}{\raggedright #1}
}
\algnewcommand{\Initialize}[1]{%
  \State \textbf{initialization}
  \Statex \hspace*{\algorithmicindent}\parbox[t]{.95\linewidth}{\raggedright #1}
}

\algnewcommand{\Outputs}[1]{%
  \State \textbf{output:}
   \parbox[t]{.8\linewidth}{\raggedright #1}
}

\begin{document}
\title{Distributed Variational Representation Learning}
\author{ I\~naki Estella Aguerri  \qquad \quad Abdellatif Zaidi \vspace{-5mm}\\   
\thanks{The results in this paper have been partially presented in \cite{ EZ:IZS:2018}. I. Estella Aguerri is with the Mathematics and Algorithmic Sciences Lab., Paris Research Center, Huawei Technologies France, Boulogne-Billancourt, 92100. A. Zaidi is with Universit\'e Paris-Est, Champs-sur-Marne, 77454, France, and the Mathematics and Algorithmic Sciences Lab., Paris Research Center, Huawei Technologies France, Boulogne-Billancourt, 92100. Emails: inaki.estella@huawei.com, abdellatif.zaidi@u-pem.fr}}

\maketitle

\begin{abstract}
The problem of distributed representation learning is one in which multiple sources of information $X_1,\ldots,X_K$ are processed separately so as to learn  as much information as possible about some ground truth $Y$. We investigate this problem from information-theoretic grounds, through a generalization of Tishby's centralized Information Bottleneck (IB) method to the distributed setting. Specifically, $K$ encoders, $K \geq 2$, compress their observations $X_1,\ldots,X_K$ separately in a manner such that, collectively, the produced representations preserve as much information as possible about $Y$. We study both discrete memoryless (DM) and memoryless vector Gaussian data models. For the discrete model, we establish a single-letter characterization of the  optimal tradeoff between complexity (or rate) and relevance (or information) for a class of memoryless sources (the observations $X_1,\ldots,X_K$ being conditionally independent given $Y$). For the  vector Gaussian model, we provide an explicit characterization of the optimal complexity-relevance tradeoff. Furthermore, we develop a variational bound on the complexity-relevance tradeoff which generalizes the evidence lower bound (ELBO) to the distributed setting. We also provide two algorithms that allow to compute this bound: i) a Blahut-Arimoto type iterative algorithm which enables to compute optimal complexity-relevance encoding  mappings by iterating over a set of self-consistent equations, and ii) a variational inference type algorithm in which the encoding mappings are parametrized by neural networks and the bound approximated by Markov sampling and optimized with stochastic gradient descent. Numerical results on synthetic and real datasets are provided to support the efficiency of the approaches and algorithms  developed in this paper.
\end{abstract}

\IEEEpeerreviewmaketitle

%%%%%%%%%%%%%%%%%

\vspace{-5mm}
\section{Introduction}

The  problem of extracting a \textit{good} representation of data, i.e., one that makes it easier to extract useful information, is at the heart of the design of efficient machine learning algorithms. One important question, which is often controversial in statistical learning theory, is the choice of a ``good'' loss function that measures discrepancies between the true values and their estimated fits. There is however numerical evidence that models that are trained to maximize mutual information, or equivalently minimize the error's entropy, often outperform ones that are trained using other criteria, such as mean-square error (MSE),  higher-order statistics~\cite{E02, PEL00}. 
%On this aspect, we also mention Fisher's dissertation~\cite{F97} which contains investigation of the application of information theoretic metrics to blind source separation and subspace projection using Renyi's entropy, as well as what appears to be the first usage of the now popular Parzen windowing estimator of information densities in the context of learning. 
Although a complete and rigorous justification of the usage of mutual information as cost function in learning is still awaited, recently a partial explanation appeared in~\cite{JCVW15} where the authors show that, under some natural data processing property, Shannon's mutual information uniquely quantifies the reduction of prediction risk due to side information. Along the same line of work, Painsky and Wornell~\cite{PW18} show that, for binary classification problems, by minimizing the logarithmic-loss (log-loss) one actually minimizes an upper bound to any choice of loss function that is smooth, proper (i.e., unbiased and Fisher consistent) and convex. Perhaps, this justifies partially why mutual information (or, equivalently, the corresponding loss function which is the log-loss fidelity measure) is widely used in learning theory and has already been adopted in many algorithms in practice such as the \textit{infomax} criterion~\cite{L88}, the tree-based algorithm of~\cite{Q14} or the well known Chow-Liu algorithm~\cite{CL68} for learning tree graphical models, with  applications in genetics~\cite{AMPB08}, image processing~\cite{PMV03}, computer vision~\cite{VWW97} and others. The log-loss measure also plays a central role in the theory of prediction~\cite[Ch. 09]{C-BL06}, where it is often referred to as the \textit{self-information} loss function, as well as in Bayesian modeling~\cite{LC06} where priors are usually designed so as to maximize the mutual information between the parameter to be estimated and the observations. 

The goal of learning, however, is not merely to learn model parameters accurately for previously seen data. Rather, in essence, it is the ability to successfully apply rules that are extracted from previously seen data to characterize new unseen data. This is often captured through the notion of ``generalization error". The generalization capability of a learning algorithm hinges on how sensitive is the output of the algorithm to modifications of the input dataset, i.e., its \textit{stability}~\cite{BE02, SSSS10}. In the context of deep learning, it can be seen as a measure of how much the algorithm overfits the model parameters to the seen data. In fact, efficient algorithms should strike a good balance among their ability to fit training dataset and that to generalize well to unseen data. In statistical learning theory~\cite{C-BL06}, such a dilemma is reflected through that the minimization of the ``population risk" (or ``test error" in the deep learning literature) amounts to the minimization of the sum of the two terms that are generally difficult to minimize simultaneously, the ``empirical risk" on the training data and the generalization error. In order to prevent over-fitting, regularization methods can be employed, which include parameter penalization, noise injection, and averaging over multiple models trained with distinct sample sets. Although it is not yet very well understood how to optimally control model complexity, recent works~\cite{XR17,RZ15} show that the generalization error can be upper-bounded using the mutual information between the input dataset and the output of the algorithm. This result actually formalizes the intuition that the less information a learning algorithm extracts from the input dataset the less it is likely to overfit; and justifies the use of mutual information also as a regularizer term. 

A growing body of works focuses on developing learning rules and algorithms using information theoretic approaches, e.g., see \cite{Tishby99theinformation, P10, YP18, YJP18, KS18 , UE-AZ17a} and references therein. Most relevant to this paper is the Information Bottleneck (IB) method of~\cite{Tishby99theinformation} which readily and elegantly captures the above mentioned viewpoint of seeking the right balance between data fit and generalization by using the mutual information both as a cost function and as a regularizer term. Specifically, IB formulates the problem of extracting the relevant information that some signal $X \in \mathcal{X}$ provides about another one $Y \in \mathcal{Y}$ that is of interest as that of finding a representation $U$ that is maximally informative about $Y$ (i.e., large mutual information $I(U;Y)$) while being minimally informative about~$X$ (i.e., small mutual information $I(U;X)$). In the IB framework, $I(U;Y)$ is referred to as the \textit{relevance} of $U$ and $I(U;X)$ is referred to as the \textit{complexity} of $U$, where  complexity here is measured by the minimum description length (or rate) at which the observation is compressed. 
 Accordingly, the performance of learning with the IB method and the optimal mapping of the data $X$ to $U$ are found by solving the Lagrangian formulation 
\begin{equation}
\mc L^*_{\mathrm{IB},s}:=\max_{P_{U|X}} I(U;Y)-s I(U;X),\label{eq:IBCriteria}
\end{equation}
where $P_{U|X}$ is a stochastic map that assigns the observation $X$ to a latent variable $U$ from which $Y$ is inferred, and $s$ is the Lagrange multiplier. There is an implicit tradeoff between complexity and relevance: when the description length (complexity) is low, the largest relevant information  captured by $U$ from $Y$ is restricted, and the other way around. The simultaneously achievable relevance-complexity pairs $(I(U;Y),I(U;X))$ define a region characterized in~\cite{Tishby99theinformation} and the resulting optimal mappings $P_{U|X}$ are made to perform at different points of the relevance-complexity region by considering different values of $s$.

Direct optimization of~\eqref{eq:IBCriteria} to obtain the optimal mappings $P_{U|X}$ is generally challenging. Instead, a tight variational bound can be optimized, which results in the optimization of a generalized version of the~\eqref{eq:VIBCriteria}  has the evidence lower bound (ELBO)~\cite{KW2013:AutoEncoding} (and the $\beta$-VAE bound \cite{higgins2016beta}) , used, e.g., for variational inference:
\begin{align}\label{eq:VIBCriteria}
\max_{P_{U|X},Q_{Y|U},Q_U} \mathbb{E}[\log Q_{Y|U}]\!-\!s D_{\mathrm{KL}}(P_{X|U}\|Q_U),
\end{align}
where $Q_{Y|U}$ and $Q_U$ are variational approximations of the optimal decoder and the latent space, and $D_{\mathrm{KL}}(\cdot\|\cdot)$ is the Kullback-Leiber divergence. 

{Different optimization methods have been proposed to optimize~\eqref{eq:VIBCriteria}. 
An iterative Blahut-Arimoto (BA) type algorithm, that converges to a stationary solution of~\eqref{eq:IBCriteria} is proposed in~\cite{Tishby99theinformation}, which results in an algorithm  similar to the expectation maximization (EM) algorithm~\cite{DLR77} for problems with discrete \cite{Slonim:2000, Slonim20122005} and  Gaussian data \cite{GlobersonTishby:Techical:Gaussian, journals/jmlr/ChechikGTW05}.  However, this algorithm generally requires a good estimation of the data distribution or perfect knowledge of it, and becomes too complex for general high-dimensional data.  An alternative strategy, which only requires data samples, has been proposed in~\cite{CMT:16, AFDM:ICLR:2017,AS:PML:18} to overcome this limitation, based on approximating the bound~\eqref{eq:VIBCriteria} by parameterizing the encoder, decoder and prior distributions with deep neural networks (DNNs), using Monte-Carlo sampling and optimizing the DNN to maximize it. 
This approach is essentially that used for variational inference~\cite{RM14} and variational autoencoders (VAE)~\cite{KW2013:AutoEncoding}. }

 The IB approach,
%  which can be seen as a generalization of the evidence lower bound (ELBO) used to train variational auto-encoders~\cite{KW2013:AutoEncoding},
   has found remarkable applications in supervised and unsupervised learning problems such as classification, clustering and prediction. 
 Perhaps key to the analysis, and theoretical development, of the IB method is its elegant connection with information-theoretic rate-distortion problems, for it is now well known that the IB problem is essentially a remote source coding problem~\cite{DT62, Witsenhausen1975} in which the distortion is measured under logarithmic loss. Recent works show that this connection turns out to be useful for a better understanding of deep neural networks (DNN)~\cite{Shwartz-ZivT17}, the emergence of invariance and disentanglement in DNN~\cite{AS17}, the minimization of PAC-bayesian bounds on the test error~\cite{McAllester13, AS17}. Other connections, that are more intriguing, exist also with seemingly unrelated problems such as privacy and hypothesis 
testing~\cite{Caire:ISIT:2016, Tian:IT:2008, SGC2018} or multiterminal networks with oblivious relays~\cite{EZCS:IT:2017}. We close this section by mentioning that the abstract viewpoint of IB seems also instrumental towards a better understanding of the so-called \textit{representation learning} \cite{bengio2013representation}, which is an active research area in machine learning that focuses on identifying and disentangling the underlying explanatory factors that are hidden in the observed data in an attempt to render learning algorithms less dependent on feature engineering. We also point out that there exists an extensive literature on building optimal estimators of information quantities (e.g. entropy, mutual information), as well as their Matlab/Python implementations, including in the high-dimensional regime, 
 e.g.,~\cite{P03, JVYW15,  CMT:16, AFDM:ICLR:2017,AS:PML:18, peng2018variational,dai2018compressing} and references therein.

\begin{figure}[t!]
\centering
\includegraphics[width=0.695\textwidth]{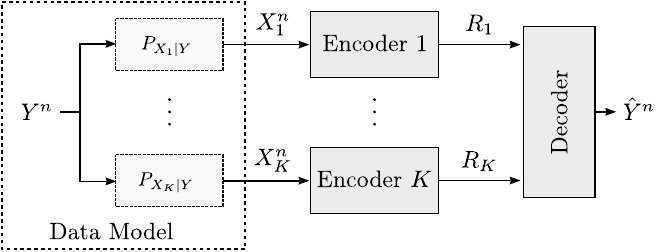}
\caption{A model for distributed learning, e.g., multi-view learning.} 
\label{fig:Schm}
\vspace{-2mm}
\end{figure}
%----------------------------------------

In this paper, we study the distributed learning problem as depicted in Figure~\ref{fig:Schm}. Here, $Y$ is the signal to be predicted and $(X_1,\hdots,X_K)$ are $K$ views of it that could each be relevant to understand one or more aspects of $Y$. The observations (or views) could be distinct or redundant. We make the assumption that $(X_1,\hdots,X_K)$ are  independent given $Y$. This assumption holds in many practical scenarios. For example,  the reader may think of $(X_1,...,X_K)$ as being independent noisy versions of $Y$. The model generalizes the aforementioned centralized learning setting to the distributed setting. By opposition to the centralized case, however, encoders should strike a suitable tradeoff between data fit and generalization \textit{collectively}, not individually. 

\vspace{-2mm}
\subsection{Example: Multi-view Learning}

In many data analytics problems, data is collected from various sources of information or feature extractors; and is intrinsically \textit{heterogenous}.  For example, an image can be identified by its color or texture features; and a document may contain text and images. Conventional machine learning approaches concatenate all available data into one big row vector (or matrix) on which a suitable algorithm is then applied. Treating different observations as a single source might cause overfitting and is not physically meaningful because each group of data may have different statistical properties. Alternatively, one may partition the data into groups according to samples homogeneity, and each group of data be regarded as a separate \textit{view}. This paradigm, termed \textit{multi-view learning}~\cite{XTX13, wang2016deep}, has received growing interest; and various algorithms exist, sometimes under references such as \textit{co-training}~\cite{BM98,DFU11,KD11,GA11}, \textit{multiple kernel learning}~\cite{GA11} and \textit{subspace learning}~\cite{JSD10}. By using distinct encoder mappings to represent distinct groups of data, and jointly optimizing over all mappings to remove redundancy, multiview learning offers a degree of flexibility that is not only desirable in practice but is likely to result in better learning capability. Actually, as Vapnik shows in~\cite{V13}, local learning algorithms produce less errors than global ones. Viewing the problem as that of function approximation, the intuition is that it is usually non-easy to find a unique function that holds good predictability properties in the entire data space. 

Besides, the distributed learning of Figure~\ref{fig:Schm} clearly finds application in all those scenarios in which learning is performed collaboratively but distinct learners either only access subsets of the data (e.g., due to physical constraints) or they access independent noisy versions of the entire dataset. Two such examples are Google Goggles and Siri in which the locally collected data is processed on clouds.

\subsection{Main Contributions}
 
We extend Tishby's centralized IB method to the distributed learning setting shown in  Figure~\ref{fig:Schm}. We study the performance of this model in the \textit{information plane}; and we characterize optimal complexity-relevance tradeoffs for a class of discrete memoryless sources as well as for Vector Gaussian sources. Furthermore,  we develop a variational bound on the optimal information-rate that can be seen as a generalization of IB method, the ELBO and the $\beta$-VAE criteria~\cite{higgins2016beta,alex2017fixing,alex2018therml} to the distributed setting. We also provide two algorithms which are trained by optimizing this bound. Specifically,  the main contributions of this paper are: 
\begin{itemize} 
\item In Section~\ref{sec:OptimalTradeoff}, we establish a single-letter characterization of optimal tradeoffs between complexity and relevance for the distributed learning model of Figure~\ref{fig:Schm} for a class of discrete sources for which the observations are independent conditionally on the target source. In doing so, we exploit the connection with the distributed Chief Executive Officer (CEO) source coding problem under logarithmic-loss distortion measure studied in~\cite{Courtade2014LogLoss}. 

\item In Section~\ref{sec:VariationalBound}, we consider the problem of maximizing relevance under a constraint on the sum complexity. We derive a variational bound which generalizes the IB cost and the ELBO  to the distributed setting and involves a novel regularization term to account for the joint complexity of the model.

\item  In Section~\ref{sec:Gauss}, we study a memoryless vector Gaussian data model. For this model, we find an explicit analytic characterization of optimal tradeoffs between complexity and relevance. The result generalizes the Gaussian information bottleneck projections~\cite{GlobersonTishby:Techical:Gaussian, journals/jmlr/ChechikGTW05} to the distributed learning scenario.

\item In Section~\ref{sec:Blahut-Algorithms} we develop two algorithms that allow to optimize the derived variational bound:
: 
1) a Blahut-Arimoto~\cite{Blahut:IT:1972}  type iterative algorithm which enables to compute optimal relevance-complexity encoding mappings by iterating over a set of self-consistent equations and that is most useful when distributions are known or can be estimated with high accuracy, and 2) a variational inference type algorithm in which the the encoders, the decoder, and the prior distributions are parametrized by DNNs and the bound approximated by Monte-Carlo sampling and optimized with stochastic gradient descent. This algorithm makes usage of Kingma et al.'s reparametrization trick~\cite{KW2013:AutoEncoding} and can be seen as a generalization of the variational information bottleneck algorithm in~\cite{AFDM:ICLR:2017} to the distributed case.  We also particularize the algorithms for the memoryless vector Gaussian data model. 

% a Blahut-Arimoto~\cite{Blahut:IT:1972} type iterative algorithm which enables to compute optimal relevance-complexity encoding mappings by iterating over a set of self-consistent equations, and a variational inference type algorithm which is found by parameterizing the encoders, the decoder, and the prior distributions via DNNs and using Monte-Carlo sampling. This algorithm makes usage of Kingma et al.'s reparametrization trick~\cite{KW2013:AutoEncoding} and can be seen as a generalization of the variational information bottleneck algorithm in~\cite{AFDM:ICLR:2017} to the distributed case.  We also particularize the algorithms for the memoryless vector Gaussian data model. 

\item  We provide numerical results on synthetic and real datasets that support the efficiency of the approaches and algorithms that are developed in this paper. 
\end{itemize}
 \vspace{-3mm}
\subsection{Notation}
Throughout, upper case letters  denote random variables, e.g., X;  lower case letters denote realizations of random variables, e.g., $x$; and calligraphic letters denote sets, e.g., $\mathcal{X}$. The cardinality of a set is denoted by $|\mc X|$. For a random variable $X$ with probability mass function (pmf) $P_{X}$, we use $P_{X}(x)=p(x)$, $x\in \mc X$ for short. Boldface upper case letters denote vectors or matrices, e.g., $\dv X$, where context should make the distinction clear.
We abbreviate a sequence $(X_1, X_2, \ldots, X_n)$ of $n$ random variables by $X_n$, and we denote the interval $(X_k, X_{k+1}, \cdots , X_j )$ by $X^j_k$. If the
lower index is equal to $1$, it will be omitted when there is no ambiguity (e.g.,
$X^j\triangleq X^j_1$).
 For random variables $(X_1,X_2,\hdots)$ and a set of integers $\mc K\subseteq \mathds{N}$, $X_{\mc K}$ denotes the set of variables with indices in  $\mc K$, i.e., $X_{\mc K}=\{X_k:k \in \mc K \}$. If $\mc K = \emptyset$, $X_{\mc K}=\emptyset$. For $k\in \mc K$ we let $X_{\mathcal{K}/k} = (X_1,\hdots,X_{k-1},X_{k+1},\hdots,X_K)$, with $X_0=X_{K+1}=\emptyset$. Also, for zero-mean random vectors $\dv X$ and $\dv Y$, the quantities $\mathbf{\Sigma}_{\mathbf{x}}$, $\mathbf{\Sigma}_{\mathbf{x},\mathbf{y}}$ and $\mathbf{\Sigma}_{\mathbf{x}|\mathbf{y}}$  denote respectively the covariance matrix of the vector $\dv X$, the covariance matric of vector $(\dv X,\dv Y)$ and the conditional covariance matrix of $\dv X$, conditionally on $\dv Y$, i.e., $\mathbf{\Sigma}_{\mathbf{x}} = \mathrm{E}[\mathbf{XX}^H]$ $\mathbf{\Sigma}_{\mathbf{x},\mathbf{y}}:= \mathrm{E}[\mathbf{XY}^H]$, and  $\mathbf{\Sigma}_{\mathbf{x}|\mathbf{y}} = \mathbf{\Sigma}_{\mathbf{x}}-\mathbf{\Sigma}_{\mathbf{x},\mathbf{y}}\mathbf{\Sigma}_{\mathbf{y}}^{-1}\mathbf{\Sigma}_{\mathbf{y},\mathbf{x}}$. Finally, for two probability measures $P_X$ and $Q_X$ on  $X \in \mc X$, the relative entropy or Kullback-Leibler divergence is denoted as $D_{\mathrm{KL}}(P_X \| Q_X)$. That is, if $P_X$ is absolutely continuous with respect to $Q_X$, $P_X \ll Q_X$, (i.e., for every $x \in \mc X$ if $P_X(x) >0$ then $Q_X(x) >0 )$, $D_{\mathrm{KL}}(P_X \| Q_X) = \mathbb{E}_{P_X}[\log(P_X(X)/Q_X(X))]$, otherwise $D_{\mathrm{KL}}(P_X \| Q_X)=\infty$.

\section{Problem Formulation}\label{sec:System}

Consider the distributed learning model shown in Figure~\ref{fig:Schm}. Let an integer $K \geq 2$ be given, and denote $\mc K=\{1,\hdots,K\}$. Also, let $(X_1,\hdots,X_K,Y) $ be a tuple of random variables which have joint probability mass function $P_{X_{\mc K},Y}(x_{\mc K},y) := P_{X_1,\hdots,X_K,Y}(x_1,\hdots,x_K,y)$ for $(x_1,\ldots, x_K) \in \mc{X}_1\times \hdots \times \mc{X}_K$ and $y \in \mc Y$, where $\mc X_k$ for $k \in \mc K$ designates the alphabet of $X_k$ and $\mc Y$ designates the alphabet of $Y$, all assumed to be finite\footnote{For simplicity, we assume finite alphabets. The results presented in this paper also extend to random variables with continuous alphabets using standard tools~\cite{elGamal:book}. Results for the vector Gaussian model, which has continuous alphabet are also provided.}. Throughout, it will be assumed that the following Markov chain holds for all $k \in \mc K$,  
\begin{align}
X_{k} \mkv Y \mkv X_{\mathcal{K}/k}, 
\label{eq:MKChain_pmf}
\end{align}
i.e., $p(x_{\mc K},y) = p(y)\prod_{k=1}^{K}p(x_k|y)$ for $x_k\in \mathcal{X}_K$, $y \in \mathcal{Y}$.

 The problem of distributed learning at hand seeks to characterize how accurate one can estimate the target variable $Y$ from the observations $(X_1,\hdots,X_K)$ when those are processed separately, each by a distinct encoder. More specifically, let a training dataset $\{(X_{1,i},\ldots, X_{K,i},Y_i)\}_{i=1}^n$ consisting of $n$ independent, identically distributed (i.i.d.) random samples drawn from the joint distribution $P_{X_{\mc K},Y}$ be given. Encoder $k \in \mc K$ only observes the sequence $X^n_k$; and processes it to generate a description $J_k=\phi_k(X_k^n)$ according to some (possibly stochastic) mapping
\begin{equation}
\phi_k: \mathcal{X}_k^n\rightarrow \mathcal{M}^{(n)}_k,
\label{encoding-mapping-encoder-k}
\end{equation}
where $\mathcal{M}^{(n)}_k$ is an arbitrary set of descriptions. The range of allowable description sets will be specified below. A (possibly stochastic) decoder $\psi(\cdot)$ collects all descriptions $J_{\mathcal{K}} = (J_1,\ldots,J_K)$ and returns an \mbox{estimate $\hat{Y}^n$ of $Y^n$ as}
\begin{equation}
\psi: \mathcal{M}^{(n)}_1\times \hdots \times \mathcal{M}^{(n)}_K \rightarrow \mathcal{\hat{Y}}^n.
\label{decoding-mapping}
\end{equation}
The accuracy of the estimation $\hat{Y}^n$ is measured in terms of the~\textit{relevance}, defined as the information that the descriptions $\phi_1(X_1^n),\hdots,\phi_K(X_K^n)$ \textit{collectively} preserve about $Y^n$, as measured by Shannon mutual information
\begin{align}
\Delta^{(n)}(P_{X_{\mc K},Y}):= \frac{1}{n} I_{P_{X_{\mc K},Y}}(Y^n;\hat{Y}^n)
\label{eq:relevance_def},
\end{align}
where $\hat{Y}^n=\psi(\phi_1(X_1^n),\hdots,\phi_K(X_K^n))$ and the subscript $P_{X_{\mc K},Y}$ indicates that the mutual information is computed under the joint distribution $P_{X_{\mc K},Y}$. 

If the encoding mappings $\{\phi_k\}_{k=1}^K$ are un-constrained, the maximization of the right hand sight (RHS) of~\eqref{eq:relevance_def} based on the given training dataset usually results in overfitting. A better generalization capability is generally obtained by constraining the \textit{complexity} of the encoders. In what follows, we do so by requiring that the range of the encoding functions at all encoders, i.e., the cardinality of the set of descriptions $\mathcal{M}^{(n)}_k$, are restricted in size. This is the so-called \textit{minimum description length} complexity measure, {often used in the learning literature to limit the  description lenght of the weights of neural networks~\cite{Hinton:1993:KNN}. A connection between the use of the minimum description complexity for limiting the description length of the input encoding as described and that of the weights is given in \cite{AS17}. }  Specifically, the encoding function $\phi_k(\cdot)$ at Encoder $k \in \mc K$ needs to satisfy  
\begin{equation}
R_k \geq \frac{1}{n} \log |\phi_k(X_k^n)| \quad \text{for all}\:\:\: X^n_k \in \mc X^n_k.
\label{eq:Complexity_definition}
\end{equation}

The characterization of the optimal performance for the distributed learning  problem that we study in this paper can be cast as that of finding the region of all simultaneously achievable relevance-complexity tuples. 
\begin{definition}
A tuple $(\Delta,R_1,\ldots, R_K)$ is said to be achievable if 
there exists a training set length~$n$, encoders $\phi_k$ for $k=1,\ldots, K$ and a decoder $\psi$, 
such that
\begin{align}
\Delta & \leq \frac{1}{n} I_{P_{X_{\mc K},Y}}\Big(Y^n; \psi(\phi_1(X_1^n),\hdots,\phi_K(X_K^n))\Big) \label{eq:DeltaConstr}\\
R_k &\geq \frac{1}{n}\log |\phi_k(X_k^n)| \quad \text{for all}\:\:\: k \in \mc K.\label{eq:RConstr}
\end{align}
The relevance-complexity region $\mathcal{RI}_{\mathrm{DIB}}$ is given by the closure of all achievable tuples $(\Delta,R_1,\ldots, R_K)$. \qed
\end{definition}

Equivalently, one may seek to characterize the mappings $\{\phi_k\}_{k=1}^{K}$ and $\psi$ that maximize the relevance, as given in the RHS of~\eqref{eq:relevance_def}, under the complexity constraint~\eqref{eq:Complexity_definition}. The function 
\begin{align}
\Delta(R_{\mc K},&P_{X_{\mc K},Y})= \max_{\{\phi_k\}_{k=1}^{K}, \psi} \Delta^{(n)}(P_{X_{\mc K},Y})\label{eq:Learning_Problem}\\
&\text{such that } n R_k\geq \log |\phi_k(X_k^n)| \quad \text{for all}\:\: k \in \mc K,
\end{align}
 gives the largest \textit{relevance} for prescribed complexity levels as measured by the rate tuple $R_{\mc K}=(R_1,\hdots,R_K)$, and will be referred to hereafter as the \textit{relevance-complexity function}.

The main the goal of this paper is to learn the encoders and decoders that achieve any relevance-complexity tuple that lies on the boundary on the region $\mc {RI}_{\mathrm{DIB}}$ of optimal relevance-complexity tradeoffs. To that end, another goal of this work is to 
characterize the region $\mc {RI}_{\mathrm{DIB}}$, for both discrete and Gaussian data models. In some cases, for the ease of the exposition, we will be content with the characterization of the relevance-complexity function $\Delta(R_{\mc K},P_{X_{\mc K},Y})$.

\section{Relevance-Complexity Tradeoffs}
%In this section we provide a single-letter characterization of the relevance-complexity function  in~\eqref{eq:Learning_Problem} by exploiting its connection with source coding under logarithmic-loss. Then, we derive a bound on the relevance-complexity function through a variational formulation of it, which expresses the learning problem as an optimization with a novel regularization term.
\subsection{Relevance-Complexity Region}\label{sec:OptimalTradeoff}

In this section we first characterize the optimal relevance-complexity region for the distributed learning model of Figure~\ref{fig:Schm}. It is now well known that the IB problem in~\eqref{eq:IBCriteria} is essentially a remote point-to-point source coding problem~\cite{DT62} in which distortion is measured under the logarithm loss (log-loss) fidelity criterion~\cite{HT07}. That is, rather than just assigning a deterministic value to each sample of the source, the decoder gives an assessment of the degree of confidence or reliability on each estimate. Specifically, given the output description $j=\phi(x^n)$ of the encoder, the decoder generates a soft-estimate $\hat{y}^n$ of $y^n$ in the form of a probability distribution over $\mc Y^n$, i.e., $\hat{y}^n=\hat{P}_{Y^n|J}(\cdot)$.  The incurred discrepancy between $y^n$ and the estimation $\hat{y}^n$ under log-loss for the observation $x^n$ is then given by     
\begin{equation}
\ell_{\mathrm{log}}(y^n,\hat{y}^n) = \frac{1}{n}\log\frac{1}{\hat{P}_{Y^n|J}(y^n|\phi(x^n))},
\label{eq:logLoss}
\end{equation}
where $\hat{P}_{Y^n|J}(y^n|\phi(x^n))$ is the value of this distribution for $y^n\in \mc Y^n$ evaluated for  $j=\phi(x^n)$, $x^n\in \mc X^n$. { The optimal tradeoff between complexity and relevance for the IB problem are characterized by the region given by the union of all relevance-complexity pairs $(\Delta, R)\in \mathds{R}_{+}^{2}$ satisfying}
\begin{align}
\Delta\leq I(Y;U)\quad\text{and}\quad R\geq I(X;U),
\end{align}
{for some pmf $P_{U|X}$ with joint distribution of the form $p_Y(y)p_{X|Y}(x|y)p_{U|X}(u|x)$, where $U$ is a latent representation of $X$. The boundary of this region is equivalent to the one described by the IB principle in \eqref{eq:IBCriteria} if solved for all $s$.}

Likewise, the distributed learning model in Figure~\ref{fig:Schm} is essentially a $K$-encoder CEO source coding problem under log-loss distortion, in which the decoder is interested in a soft estimate of $Y$ from the descriptions $J_{\mc K}$ generated from the observations $X_1,\ldots, X_K$.  A shown in Section~\ref{app:Comp_Rel_region_DM}, the equivalence of the two problems allows us to characterize the complexity-relevance region $\mc{RI}_{\mathrm{DIB}}$ through the rate-distortion region of the $K$-encoder CEO problem, which has been recently characterized in~\cite[Theorem 10]{Courtade2014LogLoss}.

\begin{theorem}\label{th:Comp_Rel_region_DM}
The relevance-complexity region $\mc {RI}_{\mathrm{DIB}}$ of the distributed learning problem with  $P_{X_{\mc K},Y}$ for which the Markov chain \eqref{eq:MKChain_pmf} holds, is given by the union of all tuples $(\Delta, R_1,\ldots, R_K)\in \mathds{R}_{+}^{K+1}$ satisfying for all $\mc S\subseteq \mc K$, 
\begin{align}\label{eq:ComplexityRelevanceFunction}
\Delta\leq \sum_{k\in \mathcal{S}} [R_k\!-\!I(X_{k};U_{k}|Y,T)]  + I(Y;U_{\mathcal{S}^c}|T),
\end{align}
for some set of pmfs $\dv P:=\{ P_{U_1|X_1,T},\ldots, P_{U_K|X_K,T}, P_T\}$
with joint distribution of the form
\begin{align}
p_T(t) p_Y(y)\prod_{k=1}^K p_{X_k|Y}(x_{k}|y)\prod_{k=1}^{K}p_{U_k|X_k,T}(u_k|x_k,t).\label{eq:distributionFactortization}
\end{align}
\end{theorem} 
\begin{proof}
The proof appears in Section~\ref{app:Comp_Rel_region_DM}.
\end{proof}

\begin{remark}
For a given joint data distribution $P_{X_{\mc K},Y}$, Theorem~\ref{th:Comp_Rel_region_DM} extends the single encoder IB principle  of Tishby~\cite{Tishby99theinformation} to the distributed learning model with $K$ encoders, which we denote by Distributed Information Bottleneck (DIB) problem. The result characterizes the optimal relevance-complexity tradeoff as a region of achievable tuples $(\Delta,R_1,\ldots, R_K)$ in terms of a distributed representation learning problem involving the optimization over $K$ conditional pmfs $P_{U_k|X_k,T}$ and a pmf $P_{T}$. The pmfs $P_{U_k|X_k,T}$ correspond to stochastic encodings of the observation $X_k$ to a latent variable, or representation, $U_k$ which captures the relevant information of $Y$ in observation $X_k$. Variable $T$ corresponds to a time-sharing among different encoding mappings, see, e.g.,~\cite{elGamal:book}. For such encoders, the optimal decoder is implicitly given by the conditional pmf of $Y$ from  $U_1,\ldots, U_K$, i.e., $P_{Y|U_{\mc K},T}$.
%, i.e.,
%\begin{align}\label{eq:DMOptimalDecoder}
%&P_{Y|U_{\mc K},T}(y_k|u_{\mc K},t)
%=\frac{ \sum_{(y,x_1,\ldots, x_K)\in {\mc Y, \mc X_{\mc K}}}P_{Y}(y) \prod_{k=1}^K P_{X_k|Y}(x_k|y)P_{U_k|X_k,T}(u_k|x_k,t) }{\sum_{(x_1,\ldots, x_K)\in \mc X_{\mc K}} \prod_{k=1}^K P_{U_k|X_k,T}(u_k|x_k,t)}. 
%\end{align}
%where
%\begin{align}
%P &= \sum_{(y,x_{\mc K})\in {\mc Y, \mc X_{\mc K}}}P_{Y}(y) \prod_{k=1}^K P_{X_k|Y}(x_k|y)P_{U_k|X_k,T}(u_k|x_k,t)\\
%P_{} &=  \sum_{(x_{\mc K})\in \mc X_{\mc K}} \prod_{k=1}^K P_{U_k|X_k,T}(u_k|x_k,t)
%\end{align}
\end{remark}

\begin{remark}
In the proof of Theorem~\ref{th:Comp_Rel_region_DM} it is shown that maximizing relevance, as given in~\eqref{eq:relevance_def}, is equivalent to minimizing the average log-loss. Also, it shows that the maximum relevance (and minimum log-loss) is achieved by employing the decoder $\psi$ that generates a soft estimation of $Y$ from the descriptions $J_{\mc K}$, given by the conditional distribution  $\psi(J_{\mc K}) = \{P_{Y^n|J_{\mc K}}(y^n|J_{\mc K})\}_{y^n\in \mc Y^n}$.
\end{remark}

%\begin{remark}\label{rem:continuous}
%The result in Theorem~\ref{th:Comp_Rel_region_DM} applies for random variables $(Y,X_1,\hdots, X_K)$ with discrete and continuous alphabets. For example, it characterizes the complexity-relevance region for classification problems in which $Y$ is discrete, e.g., the class label, and for regression problems in which $Y$ lies on a continuous space. The extension to continuous alphabets follows by standard discretization and limiting arguments of the mutual information~\cite{elGamal:book}.
%\end{remark}

\begin{remark}
By Proposition~\ref{prop:Eqreg}, it follows  that maximizing the relevance in the distributed learning model is also equivalent to maximizing the average log-likelihood of $Y$ from the descriptions $U_{\mc K}$. 
Thus, $\mc {RI}_{\mathrm{DIB}}$ also characterizes the optimal tradeoff between likelihood and description complexity in a distributed scenario with $K$ encoders in Figure~\ref{fig:Schm}.
In addition, if the destination is interested in reconstructing the observations, i.e., $Y=(X_1,\ldots, X_K)$, e.g., as done with Variational Autoencoders (VAEs) \cite{KW2013:AutoEncoding, higgins2016beta}, the region $\mc {RI}_{\mathrm{DIB}}$
 also characterizes the optimal tradeoff achievable between maximum log-likelihood and complexity for this case.
\end{remark}

\begin{remark}
{The characterization of the optimal relevance-complexity region for data models satisfying the Markov chain~\eqref{eq:MKChain_pmf} is connected to the CEO problem under logarithmic loss distortion measure, fully characterized in~\cite[Theorem 10]{Courtade2014LogLoss}. For general IB data models, i.e., without the  Markov chain~\eqref{eq:MKChain_pmf} the model connects with the distributed source coding problem under logarithmic loss, whose solution is a longstanding open problem in source coding \cite{Courtade2014LogLoss}. Note that while in some practical cases~\eqref{eq:MKChain_pmf}  might is not be satisfied, e.g., when $X_k$ are several pictures of a particular object from different angles, the proposed methods in this paper are still applicable, without the  interpretation of relevance and complexity.}
\end{remark}

\subsection{A Variational Bound}\label{sec:VariationalBound}

In this section, we consider the problem of learning the encoders and decoders that maximize relevance given a complexity constraint, i.e., that perform on the boundary of $\mc {RI}_{\mathrm{DIB}}$ for given $(R_1,\ldots,R_K)$. However, directly learning such encoders  is challenging. In what follows, we find a parameterization of the relevance-complexity region characterized in Theorem~\ref{th:Comp_Rel_region_DM} and derive a variational bound which expresses the optimal encoding and decoding mappings as the solution to an optimization of the average logarithmic loss with a novel regularization term. In  Section~\ref{sec:Blahut-Algorithms} and Section~\ref{sec:VariationalDIB} we provide algorithms to solve this optimization problem.  
For simplicity, we focus on maximizing relevance under sum-complexity constraint, i.e., $R_{\mathrm{sum}} :=\sum_{k=1}^K R_k$. 
The region of achievable relevance-complexity tuples under sum-complexity constraint is defined by:
\begin{align}%\label{eq:sumrateregion}
\mc{RI}_{\mathrm{DIB}}^{\mathrm{sum}}: = &\Big\{(\Delta, R_{\mathrm{sum}})\in \mathds{R}_+^2:\exists (R_1,\ldots, R_K)\in \mathds{R}_+^K\text{ s.t. }\nonumber
 (\Delta, R_1,\ldots, R_K)\in \mc{RI}_{\mathrm{DIB}}\text{ and }\sum_{k=1}^KR_k = R_{\mathrm{sum}}\Big\}. \nonumber
\end{align}

The region $\mc{RI}_{\mathrm{DIB}}^{\mathrm{sum}}$ can be characterized as given next.
\begin{proposition}\label{prop:Sum-RateRegion}
The relevance-complexity region  $\mc {RI}^{\mathrm{sum}}_{\mathrm{DIB}}$ is given by the convex-hull of all tuples $(\Delta, R_{\mathrm{sum}})\in \mathds{R}_{+}^{2}$ satisfying $\Delta\leq\Delta(R_{\mathrm{sum}},P_{X_{\mc K},Y})$ where
%\begin{align}
%\mc {RI}_{\mathrm{DIB}}^{\mathrm{sum}}= \textrm{convex-hull}\{(\Delta,R_{\mathrm{sum}})\in \mathds{R}_+^2:\Delta\leq \Delta( R_{\mathrm{sum}},P_{X_{\mc K},Y} )\},
%\end{align}
%where
\begin{align}
&\Delta(R_{\mathrm{sum}},P_{X_{\mc K},Y})= \max_{\dv P}  \min\left\{I(Y; U_{\mc K}),R_{\mathrm{sum}}- \sum_{k=1}^KI(X_k;U_k|Y)\right\}, \label{eq:RelevanceSumComplexityFunction}
\end{align} 
and where the maximization is over the set of pmfs  $\dv P := \{P_{U_1|X_1},\hdots, P_{U_K|X_K}\}$ such that the joint pmf factorizes as $p_Y(y)\prod_{k=1}^K p_{X_k|Y}(x_{k}|y)\prod_{k=1}^{K}p_{U_k|X_k}(u_k|x_k)
$. 
 \end{proposition}
 \begin{proof}
 The proof 
% follows by the application of the Fourier-Motzkin elimination procedure~\cite{elGamal:book} to project out the individual rates $(R_1,\ldots, R_K)$, and accounting for the Markov chains~\eqref{eq:MKChain_pmf}, and 
 is given in Section~\ref{app:Sum-RateRegion}.
 \end{proof}

The following proposition provides a characterization of the pairs $(\Delta,R_{\mathrm{sum}})$ on the boundary of $\mc {RI}^{\mathrm{sum}}_{\mathrm{DIB}}$ in terms of a parameter $s\geq 0$, as $(\Delta_{s}, R_{s})$, which is more suitable for the derivation of the variational bound. 
\begin{proposition}\label{prop:param}
For each tuple $(\Delta,R_{\mathrm{sum}})\in \mathds{R}^2_+$ on the boundary of the relevance-complexity region $\mc {RI}^{\mathrm{sum}}_{\mathrm{DIB}}$ there exist $s \geq 0$ such that $(\Delta,R_{\mathrm{sum}}) = (\Delta_{s}, R_{s})$, where
\begin{align}
&\Delta_{s} =\frac{1}{(1+s)} \left[(1+sK)H(Y)  +  s R_{s}+ \max_{\dv P}\mc L_{s}(\dv P)\right],\label{eq:Dparam}\\%[-5pt]
&R_{s} = I(Y;U_{\mc K}^*) + \sum_{k=1}^K [I(X_k;U_k^*) - I(Y;U_k^*)],\label{eq:R1param}
\end{align}
and $\dv P^*$ is the set of pmfs $\dv P$ that maximize the cost function
\begin{align}
\mc L_{s}(\dv P) :=\!- H(Y|U_{\mc K})\! - s \sum_{k=1}^K [H(Y|U_k) +I(X_k;U_k) ].\label{eq:CostF}
\end{align}
\end{proposition}
\begin{proof}
The proof appears in Section~\ref{app:param}.
\end{proof}
\begin{remark}
The optimization of the DIB cost in~\eqref{eq:CostF} generalizes the centralized Tishby's information bottleneck formulation in~\eqref{eq:IBCriteria} to the distibuted learning problem with $K$-encoders in Figure~\ref{fig:Schm}. Note that for $K = 1$ the optimization in~\eqref{eq:Dparam} reduces to the single encoder cost in~\eqref{eq:IBCriteria} with a multiplier  $s/(1+s)$.
\end{remark}

Using Proposition~\ref{prop:param} it is clear that the encoders $\{P_{U_k|X_k}\}_{k\in \mc K}$ that achieve the relevance-complexity pair $(\Delta_{s}, R_{s})$ can be computed by maximizing the regularized cost~\eqref{eq:CostF} for the corresponding value of $s\geq 0$. The corresponding optimal decoder $P_{Y|U_{\mc K}}$ for these encoders can be found as $P_{Y|U_{\mc K}}$. Different relevance-complexity pairs $(\Delta_{s}, R_{s})$ on the boundary of~$\mc {RI}_{\mathrm{DIB}}^{\mathrm{sum}}$ and the corresponding encoders and decoder achieving it can be obtained by solving~\eqref{eq:CostF} for different values of $s\geq 0$ and evaluating~\eqref{eq:Dparam} and~\eqref{eq:R1param} for the resulting solution.

The optimization of~\eqref{eq:CostF} generally requires to compute marginal distributions involving the descriptions $U_1,\hdots, U_K$, which might be costly to calculate. To overcome this limitation, in the following we derive a tight variational bound on $\mc L_s(\dv P)$ which lower bounds the DIB cost function with respect to some arbitrary distributions. Let us consider the arbitrary decoder $Q_{Y|U_{1},\dots, U_K}(y|u_1,\ldots, u_K)$ for $y\in \mc Y$, $u_1\in \mc U_1, \hdots , u_K\in \mc U_K$, the $K$ decoders $Q_{Y|U_k}(y|u_k)$ for $k\in \mc K$ for $y\in \mc Y$, $u_k\in \mc U_k$, and latent variable priors $Q_{U_k}(u_k)$,  $k\in \mc K$, $u_k\in \mc U_k$. For short, we denote
\begin{align}
\dv Q := \{Q_{Y|U_1,\ldots, U_K},Q_{Y|U_1},\ldots, Q_{Y|U_K},Q_{U_1},\ldots, Q_{U_K}\}.\nonumber
\end{align}

We define the \textit{variational DIB cost function} $\mc L^{\mathrm{VB}}_s(\dv P, \dv Q)$ as 
\begin{align}\label{eq:FunctionPQ}
&\mc L^{\mathrm{VB}}_s(\dv P, \dv Q) :=   \underbrace{\mathds{E}[\log Q_{Y|U_{\mc K}}(Y| U_{\mc K})]}_{\text{av. logarithmic-loss}}\,+\, s \underbrace{\sum_{k=1}^K\Big( \mathds{E}[\log  Q_{Y|U_k}(Y|U_k)]- D_{\mathrm{KL}}(P_{U_k|X_k}\| Q_{U_k}) \Big)}_{\text{regularizer}}.\nonumber
\end{align}
Next lemma states that $\mc L^{\mathrm{VB}}_s(\dv P, \dv Q)$ is a lower bound to $\mc L_s(\dv P)$ for all distributions $\dv Q$.
\begin{lemma}\label{lemma:QUpdate}
For fixed pmfs $\dv P$, we have
\begin{align}
\mc L_s(\dv P) \geq \mc L^{\mathrm{VB}}_s(\dv P, \dv Q), \qquad \text{for all pmfs } \dv Q.
\end{align}
In addition, there exists a unique $\dv Q$ that achieves the maximum $\max_{\dv Q}\mc L^{\mathrm{VB}}_s(\dv P, \dv Q) = \mc L_s(\dv P)$, and is given by
\begin{align}
Q^*_{U_k} = P_{U_k}, \qquad &Q^*_{Y|U_k} = P_{Y|U_k},\quad k = 1,\ldots, K, \label{eq:Qstark}\\
Q^*_{Y|U_1,\ldots,U_k} &= P_{Y|U_1,\ldots, U_K}, \label{eq:Qstarall}
\end{align}
where $P_{U_k}$, $P_{Y|U_k}$ and $P_{Y|U_1,\ldots, U_K}$ are computed from $\dv P$.
\end{lemma}

\begin{proof}
The proof appears in Section~\ref{app:QUpdate}.
\end{proof}

Using the above, the optimization in~\eqref{eq:Dparam} can be written in terms of the variational DIB cost function as follows:
\begin{align}
\max_{\dv P}\mc L_{s}(\dv P) = \max_{\dv P}\max_{\dv Q}\mc L^{\mathrm{VB}}_{s}(\dv P,\dv Q).\label{eq:VarEq}
\end{align}

\begin{remark}
The variational DIB cost function in \eqref{eq:FunctionPQ} consist of the average logarithmic loss (or cross-entropy) of estimating $Y$ form all the latent space representations $U_1,\ldots, U_K$ using the decoder $Q_{Y|U_1,\ldots, U_K}$, and a regularization term that is composed of two types of terms: i) the Kullback-Leiber divergence between encoders $P_{U_k|X_k}$ and the priors $Q_{U_k}$, which also appears in the single encoder version of the bound,  and ii)  the logarithmic loss of estimating the target variable $Y$ from each individual latent space representation $U_k$ using a decoder $Q_{Y|U_k}$, which does not appear in the single encoder case. An example of learning architecture which can be trained to minimize this cost function using neural networks is shown in Figure~\ref{fig:Latent_Variables} for $K=2$ observations.
\end{remark}

\begin{remark}
The  variational DIB  cost in \eqref{eq:FunctionPQ} is a generalization to distributed learning with $K$-encoders of the evidence lower bound (ELBO) of the target variable $Y$ given the representations $U_1,\hdots, U_K$~\cite{RM14, KW13}. If $Y =(X_1,\ldots, X_K)$, the bound generalizes the ELBO used for VAEs to the setting of $K\geq 2$ encoders. Also note that~\eqref{eq:FunctionPQ} provides an operational meaning to the $\beta$-VAE cost~\cite{higgins2016beta} with $\beta = s/(1+s)$, as a criteria to design estimators on the relevance-complexity plane, for different $\beta$ values.
\end{remark}

\subsection{Relevance-Complexity for Vector Gaussian Model}\label{sec:Gauss}

In this section we consider the memoryless vector Gaussian data model  and show that the encoders and decoders maximizing the relevance-complexity tradeoff correspond to  Gaussian distributions. In this model,  $(\mathbf{X}_1,\ldots,\mathbf{X}_K,\mathbf{Y})$ are jointly Gaussian and satisfy the Markov chain~\eqref{eq:MKChain_pmf}. Without loss in generality,
let the target variable be a zero-mean multivariate Gaussian $\mathbf{Y}\in \mathds{C}^{n_y}$, with  covariance matrix
$\mathbf{\Sigma}_{\mb y}$, i.e., $\mathbf{Y}\sim \mc{CN}(\dv y; \dv 0, \mathbf{\Sigma}_{\mb y})$.
 Encoder $k$, $k= 1,\ldots, K$, observes a noisy observation $\mb X_{k}\in \mathds{C}^{n_k}$, that is given by
\begin{equation}
\mathbf{X}_k = \mathbf{H}_{k}\mathbf{Y}+\mathbf{N}_k,
\label{mimo-gaussian-model}
\end{equation}
where $\mathbf{H}_{k}\in \mathds{C}^{n_k\times n_y}$ models the linear  model connecting $\dv Y$ to the observation at encoder $k$, and $\mathbf{N}_k\in\mathds{C}^{n_k}$, $k = 1,\ldots,K$, is the noise vector at encoder $k$, assumed to be Gaussian with zero-mean and covariance matrix $\mathbf{\Sigma}_{k}$, and independent from all other noises and  $\dv Y$. Note that it can be shown that for every jointly Gaussian random vector $(\mathbf{X}_1,\ldots,\mathbf{ X}_K,\mathbf{ Y})$ satisfying~\eqref{eq:MKChain_pmf}, there exist matrices $(
\mathbf{H}_1,\ldots,\mathbf{H}_K)$  and noises $\mathbf{N}_1,\ldots, \mathbf{N}_K$ as defined.

%\subsection{Relevance-Complexity Region for Vector Gaussian Model}\label{ssec:GaussRegion}
The vector Gaussian model satisfies the Markov chain~\eqref{eq:MKChain_pmf}; and thus, 
the result of Theorem~\ref{th:Comp_Rel_region_DM}, which can be extended  to continuous sources using standard techniques~\cite{elGamal:book}, characterizes the relevance-complexity region of this model. Let us denote by $\mc {RI}_{\mathrm{DIB}}^{\mathrm{G}}$ the relevance-complexity region of the Gaussian model~\eqref{mimo-gaussian-model}. Next theorem characterizes  $\mc {RI}_{\mathrm{DIB}}^{\mathrm{G}}$, and shows that there is not need for time-sharing, i.e., $T=\emptyset$ and that the optimal stochastic mappings $P_{U_k|X_k}$, $k\in \mc K$, can be restricted without loss in optimality to be multivariate Gaussian distributions as, 
\begin{align}\label{eq:GaussMap}
\dv U_k = \dv A_k \dv X_k + \dv Z_k\sim \mc {CN}(\dv u_k;\dv A_k \dv X_k, \dv \Sigma_{z,k}),
\end{align}
where $\dv A_k \in \mathds{C}^{n_k\times n_k}$ projects the observation $\dv X_k$ and $\dv Z_k$ is a zero-mean Gaussian noise with covariance $\dv \Sigma_{z,k}$. 
\begin{theorem}
\label{th:GaussSumCap}
The relevance-complexity region $\mc {RI}_{\mathrm{DIB}}^{\mathrm{G}}$ for the vector Gaussian model is given by the union of all tuples $(\Delta, R_1,\ldots,R_L)$ satisfying for all $\mathcal{S} \subseteq \mathcal{K}$
\begin{align}
\Delta&\leq \sum_{k\in\mathcal{S}}\left(R_k+\log\left|\dv I-\mathbf{\Sigma}_{k}^{1/2}\mathbf{\Omega}_{k}\mathbf{\Sigma}_{k}^{1/2}\right|\right)  + \log\left| \mathbf{I} +\sum_{k\in\mathcal{S}^{c}}\mathbf{\Sigma}_{\mb y}^{1/2}\mathbf{H}_{k}^{\dagger}
\mathbf{\Omega}_{k}
\mathbf{H}_{k}\mathbf{\Sigma}_{\mb y}^{1/2}\right|
,\nonumber%\label{eq:GaussSumCap}
\end{align}
for some  $\dv 0\preceq\dv \Omega_k\preceq\dv \Sigma_k^{-1}$. In addition, the relevance-complexity region is achievable with $T= \emptyset$ and pmfs
\begin{align}
P^{*}_{U_k|X_k,T}(\dv u_k|\dv x_k,t) = \mc{CN}(\dv u_k; \dv x_k,  \mathbf{\Sigma}_{k}^{1/2}(\dv \Omega_k-\dv I)\mathbf{\Sigma}_{k}^{1/2}  ).\nonumber
\end{align}
\end{theorem}
\begin{proof}
The proof is given in Section~\ref{app:GaussSumCap}.
\end{proof}

\begin{remark}
Theorem~\ref{th:GaussSumCap} extends the result of~\cite{GlobersonTishby:Techical:Gaussian}, \cite{journals/jmlr/ChechikGTW05} and \cite{winkelbauer2014rate} on the relevance-complexity tradeoff characterization of the single-encoder IB problem for jointly Gaussian sources~\eqref{eq:MKChain_pmf} to $K$ encoders.
\end{remark}
{Note that, from~\eqref{eq:GaussMap} and the proof of Thereom~\ref{th:GaussSumCap}, $ \mathbf{\Omega}_k$, $k=1,\ldots, K$ can be understood as parameterizations of the  covariance matrixs of the noises $\dv\Sigma_{z,k}$ in \eqref{eq:GaussMap}.}
By Theorem~\ref{th:GaussSumCap}, the optimal encoders $P_{U_k|X_k}$ are multivariate Gaussian distributions as in~\eqref{eq:GaussMap}. Thus, the decoder $P_{Y|U_1,\ldots, U_K}$ is also a multivariate Gaussian distribution, since any conditional distribution of jointly Gaussian variables is also a Gaussian distribution. Similarly, the optimal distributions  $\dv Q$ in Lemma~\ref{lemma:QUpdate} for given multivariate Gaussian distributions $\dv P$ are multivariate Gaussian distributions.
\begin{corollary}\label{coro:GaussPmf}
If $(\dv X_1,\ldots, \dv X_K,\dv Y)$ are jointly Gaussian as in \eqref{mimo-gaussian-model}, for given encoders $\dv P$ with multivariate Gaussian distribution as in~\eqref{eq:GaussMap}, the optimal distributions $\dv Q$ in \eqref{eq:Qstark} and \eqref{eq:Qstarall} are multivariate Gaussian distributions.
\end{corollary}

For the vector Gaussian model, the optimal encoding and decoding mappings can be computed explicitly. However, in general  closed form expressions are unfeasible. In the following sections we provide algorithms to find mappings that approximate these optimal distributions.

\section{Blahut-Arimoto DIB Algorithm}\label{sec:Blahut-Algorithms}

In this section, we derive an iterative method to optimize the variational DIB cost function in~\eqref{eq:VarEq} when the data model is discrete and the joint distribution $P_{X_{\mc K},Y}$ is either known, or a good estimation of it can be obtained from the training samples. In these cases, the maximizing distributions $\dv P,\dv Q$ of the variational DIB cost in~\eqref{eq:VarEq} can be efficiently found by an alternating optimization procedure over $\dv P$ and 
$\dv Q$ similar to the expectation-maximization (EM) algorithm~\cite{DLR77} and the standard Blahut-Arimoto (BA) method\cite{Blahut:IT:1972}. An extension to the vector Gaussian data model, which involves random variable with continuous alphabets, is also provided. The main idea of the algorithm is that at iteration $t$, the optimal distributions $\dv P^{(t)}$ that maximize the variational D-IB bound $\mc L^{\mathrm{VB}}_s(\dv P,\dv Q^{(t)})$ for fixed $\dv Q^{(t)}$ can be optimized in closed form and, next, the maximizing pmfs $\dv Q^{(t)}$ for given $\dv P^{(t)}$ can be also found analytically. So, starting from an initialization $\dv P^{(0)} $ and $\dv Q^{(0)}$ the algorithms performs the following computations successively, until convergence,
\begin{align}
\dv P^{(0)}\rightarrow \dv Q^{(0)}\rightarrow \dv P^{(1)}\rightarrow \hdots \rightarrow \dv P^{(t)} \rightarrow \dv Q^{(t)}\rightarrow \hdots
\end{align}

We refer to such algorithm as ``Blahut-Arimoto Distributed Information Bottleneck Algorithm (BA-DIB)''.
Algorithm~\ref{algo:BA_DMC} describes the steps taken by BA-DIB to successively maximize $\mc L^{\mathrm{VB}}_{s}(\dv P,\dv Q)$ by solving a concave optimization problem over $\dv P$ and over $\dv Q$ at each iteration. 

\begin{lemma}\label{lem:convex}
The function $\mc L^{\mathrm{VB}}_s(\dv P, \dv Q)$ is concave in $\dv P$ and in $\dv Q$.
 \end{lemma}

\begin{proof}
The proof follows by using the log-sum inequality~\cite{Cover:book} and the convexity of the mapping $x\mapsto x\log x$; and it is omitted for brevity.
%Let $\dv P_1, \dv P_2$ and $\dv Q_1, \dv Q_2$ and $\dv P = \lambda \dv P_1 + \bar{\lambda}\dv P_2$ and $\dv Q = \lambda \dv Q_1 + \bar{\lambda}\dv Q_2$  for $\lambda\in [0,1]$ and $\bar{\lambda} := 1-\lambda$. 
%The convexity of $I(U_k;Y_k)$ is well known~\cite{Cover:book}. For the other terms in $F_s(\cdot)$, we have from the log-sum inequality
%\begin{align}
%p \log \frac{q}{p}&\leq \lambda p_1\log \frac{\lambda q_1}{\lambda p_2} +\bar{\lambda}p_2\log\frac{\bar{\lambda q_2}}{\bar{\lambda p_2}}\\
%&=\lambda p_1\log q_1 +\bar{\lambda}p_2\log q_2 -\lambda p_1 \log p_1 -\bar{\lambda}\log p_2.\nonumber
%\end{align}
%It follows that $p \log q\leq \lambda p_1\log q_1 +\bar{\lambda}p_2\log q_2$ provided 
%\begin{align}
%p\log p -\lambda p_1 \log p_1 -\bar{\lambda}\log p_2\leq 0,
%\end{align}
%which is true due to the the convexity of $x\log x$. Applying this relation to the elements of $Q_{X|U_k}$ and $Q_{X|U_{\mc K}}$, and summing over all $x\in \mc X$ and $u_k\in \mc U_k$ completes the proof.
\end{proof}

For fixed $\dv P^{(t)}$, the optimal $\dv Q^{(t)}$ maximizing the variational D-IB bound in~\eqref{eq:FunctionPQ} follows from Lemma~\ref{lemma:QUpdate} as given by \eqref{eq:Qstark}-\eqref{eq:Qstarall}. For fixed $\dv Q^{(t)}$, the optimal $\dv P^{(t)}$ can be found using the following lemma.

\begin{lemma}
For fixed $\dv Q$, there exists a  $\dv P$ that achieves the maximum $\max
_{\dv P}\mc L^{\mathrm{VB}}_s(\dv P, \dv Q)$, where $P_{U_k|X_k}$ is given by
\begin{align}
p^*(u_k|x_k) = q(u_k)\frac{\exp\left(-\psi_s(u_k,x_k)\right)}{\sum_{u_k\in \mc U_k}q(u_k) \exp(-\psi_s(u_k,x_k))},\label{eq:P_update}
\end{align}
for $u_k\in \mc U_k$ and $x_k\in \mc X_k$, $k\in \mc K$,  and where we define 
\begin{align}\label{eq:P_update_psi}
\psi_s(u_k,x_k)&:=  D_{\mathrm{KL}}(P_{Y|x_k}||Q_{Y|u_k})+\frac{1}{s}
\mathds{E}_{U_{\mc K\setminus k}|x_k}[D_{\mathrm{KL}}(P_{Y|U_{\mc K\setminus k},x_k}||Q_{Y|U_{\mc K\setminus k},u_k}))].
%\nonumber
\end{align}
\end{lemma}
%=============================================================================================================================================

\begin{proof}
Due to its concavity,  to maximize  $\mc L^{\mathrm{VB}}_s(\dv P, \dv Q)$  with respect to $\dv P$ for given $\dv Q$,  we add the Lagrange multipliers $\lambda_{x_k}\geq 0$  for each constraint $\sum_{u_k\in \mc U_k}p(u_k|x_k) = 1$ with $x_k\in \mc X_k$.  For each $s$, $\lambda_{x_k}\geq 0$ and $p(u_k|x_k)$ can be explicitly found by solving the KKT conditions, e.g., 
{\small
\begin{align}
\frac{\partial}{\partial p(u_k|x_k)}\left[\mc L^{\mathrm{VB}}_s(\dv P,\dv Q) \!+\!\!\sum_{x_k\in \mc X_k}\lambda_{x_k}\!\!\left(\sum_{u_k\in \mc U_k}p(u_k|x_k) - 1\right)\right] = 0.\nonumber
\end{align}}
This completes the proof.
\end{proof}

\begin{algorithm}[t!]
\caption{BA-DIB training algorithm for discrete data}\label{algo:BA_DMC}
\begin{algorithmic}[1]
\smallskip
\Inputs{discrete pmf $P_{X_1,\ldots,X_k,Y}$, parameter $s\geq 0$.}
\Outputs{optimal $P^*_{U_k|X_k}$, pair $(\Delta_s,R_s)$.}
\Initialize{Set $t=0$ and set $\dv P^{(0)}$ with $p(u_k|x_k)= \frac{1}{|\mc U_k|}$ \\ for $u_k\in \mc U_k$, $x_k\in \mc X_k$, $k=1,\ldots, K$.}
%\smallskip
\Repeat 
\State Compute $\dv Q^{(t+1)}$ using \eqref{eq:Qstark} and \eqref{eq:Qstarall}. %from $\dv P^{(t)}$.
\State Compute $\dv P^{(t+1)}$ using \eqref{eq:P_update}. %from $\dv Q^{(t+1)}$ and $\dv P^{(t)}$.
\State $t \leftarrow t+1$.
\Until{convergence.}
\end{algorithmic}
\end{algorithm}

After convergence of the encoders $\dv P^{*}$ and decoders $\dv Q^*$, the target variable $Y$ can be inferred for a new observation using $P^*_{U_K|X_k}(U_k|X_k)$ and the soft estimate provided by $Q_{Y|U_1,\ldots, U_K}^*(Y|U_1,\ldots, U_K)$.

\subsection{Convergence of the BA-DIB Algorithm}
 Algorithm \ref{algo:BA_DMC} essentially falls into the class of  the Successive Upper-Bound Minimization (SUM) algorithms \cite{Razaviyayn:SIAM:UnifiedConvergence}  in which $\mc L^{\mathrm{VB}}_s(\dv P, \dv Q)$ acts as a globally tight lower bound on $\mc L_s(\dv P)$. Algorithm \ref{algo:BA_DMC} provides a sequence $\dv P^{(t)}$ for each iteration $t$, which converges to a stationary point of the problem in~\eqref{eq:VarEq}. 
\vspace{-1mm}
\begin{proposition}
Every limit point of the sequence $\dv P^{(t)}$ generated by Algorithm~\ref{algo:BA_DMC} converges to a stationary point of~\eqref{eq:VarEq}.
\end{proposition}

\begin{proof}
Let $\dv Q^*(\dv P):= \arg\max_{\dv Q}\mc L^{\mathrm{VB}}_s(\dv P, \dv Q)$. From Lemma~\ref{lemma:QUpdate},  $\mc L^{\mathrm{VB}}_s(\dv P, \dv Q^*(\dv P'))\leq \mc L^{\mathrm{VB}}_s(\dv P, \dv Q^*(\dv P)) = \mc L_s(\dv P) $ for $\dv P'\neq \dv P$.  Since $\mc L_s(\dv P)$ and $\mc L^{\mathrm{VB}}_s(\dv P, \dv Q^*(\dv P'))$ satisfy \cite[Proposition 1]{Razaviyayn:SIAM:UnifiedConvergence},  then $\mc L^{\mathrm{VB}}_s(\dv P, \dv Q^*(\dv P'))$  satisfies A1-A4 in \cite{Razaviyayn:SIAM:UnifiedConvergence}. Convergence to a stationary point of~\eqref{eq:VarEq} is due to \cite[Theorem 1]{Razaviyayn:SIAM:UnifiedConvergence}.
\end{proof}

\begin{remark}
The resulting set of self consistent equations  \eqref{eq:Qstark}, \eqref{eq:Qstarall} and \eqref{eq:P_update_psi} satisfied by any stationary point of the D-IB problem extend those of the standard point-to-point IB problem~\cite{GlobersonTishby:Techical:Gaussian} to the distributed IB problem with $K\geq 2$ encoders. The resulting equations are reminiscent of those for the point-to-point IB problem with an additional divergence term in  \eqref{eq:P_update_psi} for encoder $k$ averaged over the descriptions at the other $\mc K\setminus k$ encoders.
\end{remark}

\subsection{BA-DIB Algorithm for the Vector Gaussian Model }\label{ssec:GaussBA}
Computing the relevance-complexity for the vector Gaussian model from Theorem~\ref{th:GaussSumCap} is a convex optimization problem on $\dv \Omega_k$, which can be efficiently solved with generic tools. 
In the following, we extend Algorithm~\ref{algo:BA_DMC} as an alternative method to maximize relevance under sum-complexity constraint for vector Gaussian models.

For finite alphabet sources the updating rules of  $\dv Q^{(t+1)}$ and $\dv P^{(t+1)}$ in Algorithm~\ref{algo:BA_DMC} are relatively easy, but they become unfeasible for continuous alphabet sources. We leverage on the optimality of Gaussian encoders, shown in  Theorem~\ref{th:GaussSumCap}, to restrict the optimization of $\dv P$ to Gaussian distributions, which are  represented by a finite set of parameters, namely its mean and covariance. We show that if $\dv P^{(t)}$ are Gaussian distributions, then $\dv{P}^{(t+1)}$ are also Gaussian distributions, and can be computed with an efficient update algorithm of its representing parameters. In particular, if at time $t$ the $k$-th distributions $P_{\dv U_k|\dv X_k}^{(t)}$ is given by
\begin{align}
\dv U_k^{t} = \dv A_{k}^{t}\dv X_k +\dv Z_{k}^{t},\label{eq:testChan}
\end{align}
where $\dv Z_{k}^{t}\sim\mc{CN}(\dv 0,\dv \Sigma_{\dv z_{k}^{t}})$, we show that at $t+1$, for $\dv P^{(t+1)}$  updated as in \eqref{eq:P_update}, the encoder $P_{\dv U_{k}|\dv X_k}^{(t+1)}$ corresponds to $\dv U_{k}^{t+1} = \dv A_{k}^{t+1}\dv X_{k}+\dv Z_{k}^{t+1}$,
where $\dv Z_{k}^{t+1}\!\sim\mc{CN}(\dv 0, \dv\Sigma_{\dv z_{k}^{t+1}})$  and  $\dv\Sigma_{\dv z_{k}^{t+1}}, \dv A_{k}^{t+1}$ are updated as
\begin{align}
\dv \Sigma_{\dv z_k^{t+1}} &=\left(\left(1+\frac{1}{s}\right)\dv \Sigma_{\dv u_k^t|\dv y}^{-1}  - \frac{1}{s} \dv \Sigma_{\dv u_k^t|\dv u_{\mc K\setminus k}^t}^{-1}\right)^{-1},\label{eq:SigmaUpdate}\\
\dv A_{k}^{t+1} &=\dv \Sigma_{\dv z_k^{t+1}} \left(\left(1+\frac{1}{s}\right)\dv \Sigma_{\dv u_k^t|\dv y}^{-1}\dv  A_{k}^t(\dv I - \dv \Sigma_{\dv x_k|\dv y}\dv \Sigma_{\dv x_k}^{-1})
\right.\left.
-\frac{1}{s}\dv \Sigma_{\dv u_k^t|\dv u_{\mc K\setminus k}^t}^{-1}\dv  A_{k}^t(\dv I - \dv \Sigma_{\dv x_k|\dv u_{\mc K\setminus k}^t}\dv \Sigma_{\dv x_k}^{-1})\right).\label{eq:AUpdate}
\end{align}
The detailed update procedure is given in Algorithm~\ref{algo:BA_Gauss}, and the details of the derivation  are relegated to Section~\ref{app:BADIVGauss}.

\begin{remark}
Algorithm~\ref{algo:BA_Gauss} generalizes the iterative  algorithm  for single encoder Gaussian IB in~\cite{journals/jmlr/ChechikGTW05}  to the Gaussian D-IB  with $K$ encoders and sum-complexity constraint. Similarly to the solution in~\cite{journals/jmlr/ChechikGTW05}, the optimal description at each encoder is given by a noisy linear projection of the observation whose dimensionality is determined by the regularization parameter $s$ and the second order moments between the observed data and the target variable, as well as a term depending on the  observed data with %respect to the descriptions at the other encoders.
\end{remark}

%=============================================================================================================================================
\begin{algorithm}[!t]
\caption{BA-DIB algorithm for the Gaussin Vector D-IB}\label{algo:BA_Gauss}
\begin{algorithmic}[1]
\smallskip
\Inputs{covariance ${\dv \Sigma}_{\dv y, \dv x_1,\ldots,\dv x_k}$, parameter $s\geq 0$.}
\Outputs{optimal pairs $(\dv A_k^{*},\dv \Sigma_{\dv z_k^{*}} )$, $k=1,\ldots, K$.}
\Initialize{Set randomly $\dv A_k^{0}$ and $\dv \Sigma_{\dv z_k^{0}} \succeq 0$, $k\in \mc K$.}
%\smallskip
\Repeat 
\State Compute $\dv \Sigma_{\dv x_k|\dv u_{\mc K\setminus k}^t}$ and update for $k\in \mc K$
\begin{align}
\dv \Sigma_{\dv u_k^t|\dv y} &= \dv A^{t}_k \dv \Sigma_{\dv x_k|\dv y} \dv A^{t,\dagger}_k + \dv \Sigma_{\dv z_k^t}\label{eq:Cov_ux}\\
\dv \Sigma_{\dv u_k^t|\dv u_{\mc K\setminus k}^t} &= \dv A^{t}_k \dv \Sigma_{\dv x_k|\dv u_{\mc K\setminus k}^t} \dv A^{t,\dagger}_k + \dv \Sigma_{\dv z_k^t},\label{eq:Cov_uus}
\end{align}

\State Compute $\dv \Sigma_{\dv z_k^{t+1}}$ as in \eqref{eq:SigmaUpdate} for $k\in \mc K$.
\State Compute $\dv A_k^{t+1}$ as \eqref{eq:AUpdate}, $k\in \mc K$.
\State $t \leftarrow t+1$.
\Until{convergence.}
\end{algorithmic}
\end{algorithm}
%=============================================================================================================================================
\vspace{-2mm}

 \section{Variational Distributed  IB Algorithm}\label{sec:VariationalDIB}

In cases in which the joint distribution of the data is either known perfectly or can be estimated to high accuracy, the maximizing distributions $\dv P,\dv Q$ of the variational DIB cost in \eqref{eq:VarEq} can be found through BA-DIB algorithm as shown in Section~\ref{sec:Blahut-Algorithms}. However, in general only a set of training samples $\{(X_{1,i},\ldots, X_{K,i},Y_i)\}_{i=1}^n$ are available. In this section, we provide a method to optimize~\eqref{eq:VarEq} in this situation by parameterizing the encoding and decoding distributions that are to optimize using a family of distributions whose parameters are determined by DNNs. This allows us to formulate~\eqref{eq:VarEq} in terms of the DNN parameters and optimize it by using the reparameterization trick~\cite{KW13}, Monte Carlo sampling, as well as stochastic gradient descent (SGD) type algorithms. {The proposed method generalizes the variational framework in~\cite{CMT:16, AFDM:ICLR:2017,AS:PML:18, peng2018variational,dai2018compressing}  to the scenario with $K$ encoders in Figure~\ref{fig:Schm}.}

{Let $P_{\theta_k}(u_k|x_k)$ denote the family of encoding probability distributions $P_{U_k|X_k}$ over $\mc U_k$ for each element on $\mc X_k$, parameterized by the output of a DNN $f_{\theta_k}$ with parameters $\theta_k$.
A common example  is the family of multivariate Gaussian distributions, \cite{KW13,AFDM:ICLR:2017}, which are parameterized by the mean $\boldsymbol{\mu}_k^{\theta}$ and covariance matrix $\dv \Sigma_k^{\theta}$, i.e., $\boldsymbol\gamma_k := (\boldsymbol{\mu}_k^{\theta}, \dv \Sigma_k^{\theta})$.  Given an observation $X_k$, the values of $(\boldsymbol{\mu}_k^{\theta}, \dv \Sigma_k^{\theta})$ are determined by the output of the DNN $f_{\theta_k}$ and the corresponding family member is given by $P_{\theta_k}(u_k|x_k) = \mc {N} (u_k; \boldsymbol{\mu}_k^{\theta}, \dv \Sigma_k^{\theta})$.  For discrete distributions, a common example are concrete variables \cite{JGP17} (or Gumbel-Softmax \cite{MMT16}). More expressive paremeterizations can also be considered, e.g.,~\cite{KSW16, PPM17}.}

{Similarly, for decoders $Q_{Y|U_k}$ over $\mc Y$ for each element on $\mc U_k$, and the decoding distribution $Q_{Y|U_{\mc K}}$ over $\mc Y$ for each element in $ \mc U_1\times \cdots \times \mc U_K$, let $Q_{\psi_k}(y|u_k)$, $Q_{\psi_{\mc K}}(y|u_{\mc K})$ denote the families of distributions parameterized by the output of the DNNs $f_{\psi_k}, f_{\psi_{\mc K}}$, respectively. Finally, for the prior distributions $Q_{U_k}(u_k)$ over $\mc U_k$ we define the family of distributions $Q_{\varphi_k}(u_k)$, which do not depend on a DNN.}

By restricting the optimization of the variational DIB cost in~\eqref{eq:VarEq} to the encoder, decoder and priors within the families of distributions  $P_{\theta_k}(u_k|x_k)$, $Q_{\psi_k}(y|u_k)$, $Q_{\psi_{\mc K}}(y|u_{\mc K})$ and $Q_{\varphi_k}(u_k)$ we get
\begin{align}
\max_{\dv P}\max_{\dv Q}\mc L^{\mathrm{VB}}_{s}(\dv P,\dv Q)\geq \max_{\boldsymbol{\theta},\boldsymbol{\phi}, \boldsymbol\varphi}\mc L^{\mathrm{NN}}_s({\boldsymbol\theta, \boldsymbol\phi, \boldsymbol \varphi} ),\label{eq:NNCost}
\end{align}
where we use the notation $\boldsymbol{\theta}:=[\theta_1,\ldots,\theta_K]$, $\boldsymbol{\phi}:=[\phi_1,\ldots,\phi_K, \phi_{\mc K}]$ and $\boldsymbol \varphi:=[\varphi_1,\ldots, \varphi_K]$ to denote the DNN and prior parameters and, the cost in~\eqref{eq:NNCost} is given by
\begin{align}
\mc L^{\mathrm{NN}}_s({\boldsymbol\theta, \boldsymbol\phi, \boldsymbol \varphi} ):=
&\mathds{E}_{P_{Y,X}}\mathds{E}_{\{P_{\theta_k}(U_k|X_k)\}}\Big[\log Q_{\phi_{\mc K}}(Y| U_{\mc K})
\nonumber\\
&+s\sum_{k=1}^K\Big(\log Q_{\phi_{k}}(Y| U_{k})- D_{\mathrm{KL}}(P_{\theta_k}(U_{k}|X_{k})\|Q_{\varphi_k}(U_k)) \Big)\Big]\nonumber.
\end{align}

\begin{algorithm}[!t]
\caption{D-VIB training algorithm for the D-IB problem}\label{algo:VDIB}
\begin{algorithmic}[1]
\smallskip
\Inputs{Training Dataset $\mc D$, parameter $s\geq 0$, DNN parameters $\boldsymbol{\theta},\boldsymbol{\phi}, \boldsymbol\varphi$.}
\Outputs{optimized $\boldsymbol{\theta},\boldsymbol{\phi}, \boldsymbol\varphi$ and pair $(\Delta_s,R_s)$.}
\Initialize{Initialize $\boldsymbol{\theta},\boldsymbol{\phi}, \boldsymbol\varphi$ and set $t=0$.}
%\smallskip
\Repeat 
\State Randomly select $b$ minibatch samples $X_{\mc K}^b= (X_1^b,\ldots,  X_K^b)$ and the corresponding $Y^b$ form $\mc D$.
\State Draw $m$ random samples $Z_k^m\sim P_{Z_k}$, $k=1,\ldots, K$.
\State Compute $m$ samples $U_{k,l}^b = g_{\phi_k}(X^b,Z_{k,l}^b)$.
\State Compute gradients of the empirical DIB cost in~\eqref{eq:Lhatoptimization},  $\nabla_{\boldsymbol\theta, \boldsymbol\phi, \boldsymbol\varphi} \sum_{i=1}^b\mathcal{L}^{\mathrm{emp}}_{s,i}({\boldsymbol\theta, \boldsymbol\phi, \boldsymbol \varphi} )$  for $(X_{\mc K}^b,Y^b)$.
\State Update $(\boldsymbol\theta, \boldsymbol\phi, \boldsymbol \varphi)$ using the estimated gradient (e.g. with SGD or ADAM).
\Until{convergence of $(\boldsymbol\theta, \boldsymbol\phi, \boldsymbol \varphi)$.}
\end{algorithmic}
\end{algorithm}

Next, we train the DNNs to maximize a Monte Carlo approximation of~\eqref{eq:NNCost} over ${\boldsymbol\theta, \boldsymbol\phi, \boldsymbol \varphi}$ using optimization methods such as SGD or ADAM~\cite{KDB14} with backpropagation.  We use the reparameterization trick~\cite{KW13}, to sample from $P_{\theta_k}(U_k|X_k)$. In particular, we consider $P_{\theta_k}(U_k|X_k)$. to belong to a parametric family of distributions that can be sampled by first sampling a random variable $Z_k$ with distribution $P_{Z_k}(z_k)$, $z_k\in \mc Z_k$ and then transforming the samples using some function $g_{\theta_k}:\mc X_k\times \mc Z_k\rightarrow \mc U_k$ parameterized by $\theta_k$, such that $U_k = g_{\theta_k}(x_k,Z_k)\sim P_{\theta_k}(U_k|x_k)$.  Various parametric families of distributions fall within this class both for discrete and continuous latent spaces, e.g., the Gumbel-Softmax distributions and the Gaussian distributions mentioned above.
The reparametrization trick reduces the original optimization to estimating $\theta_k$ of the deterministic function $g_{\theta_k}$ and allows to compute estimates of the gradient  using backpropagation~\cite{KW13}. The variational DIB cost in~\eqref{eq:NNCost} can be approximated, by sampling  $m$ independent samples $\{u_{k,i,j}\}_{j=1}^m \sim P_{\theta_k}(u_{k}|x_{k,i})$ for each training sample $(x_{1,i},\ldots, x_{K,i},y_i)$, $i=1, \hdots, n$. Sampling is performed by using $u_{k,i,j} = g_{\phi_k}(x_{k,i},z_{k,j})$ with $\{z_{k,j}\}_{j=1}^m$ i.i.d. sampled from $P_{Z_k}$. Altogether,  we have the empirical DIB cost for the $i$-th sample in the training dataset:
\begin{align}
\mathcal{L}^{\mathrm{emp}}_{s,i}({\boldsymbol\theta, \boldsymbol\phi, \boldsymbol \varphi} )
&\!:=\!\frac{1}{m}\sum_{j=1}^m\Big[\!\log Q_{\phi_{\mc K}}(y_i| u_{1,i,j},\ldots, u_{K,i,j} )\label{eq:EmpiricalL}
%\nonumber
\\
&\!\!+ s\!\sum_{k=1}^K \!\Big(\! \log Q_{\phi_{k}}(y_i|u_{k,i,j})\!-\!D_{\mathrm{KL}}(P_{\theta_k}(U_{k,i}|x_{k,i})\|Q_{\varphi_k}(U_{k,i}))
\Big)\Big].
\nonumber
\end{align}

Finally, we maximize the empirical DIB cost over the DNN parameters  $\boldsymbol{\theta},\boldsymbol{\phi},\boldsymbol\varphi$ over the training data as,
\begin{align}
 \max_{\boldsymbol{\theta},\boldsymbol{\phi}, \boldsymbol\varphi}\frac{1}{n}\sum_{i=1}^n\mathcal{L}^{\mathrm{emp}}_{s,i}({\boldsymbol\theta, \boldsymbol\phi, \boldsymbol \varphi} ).\label{eq:Lhatoptimization}
\end{align}

By the law of large numbers, for large $n,m$, we have $1/n\sum_{i=1}^n\mathcal{L}^{\mathrm{emp}}_{s,i}({\boldsymbol\theta, \boldsymbol\phi, \boldsymbol \varphi} ))\rightarrow \mc L^{\mathrm{NN}}_s({\boldsymbol\theta, \boldsymbol\phi, \boldsymbol \varphi} )$ almost surely.

The training of the D-VIB algorithm is detailed in Algorithm~\ref{algo:VDIB}. After convergence of the DNN parameters  to $\boldsymbol{\theta}^*,\boldsymbol{\phi}^*, \boldsymbol\varphi^*$,  for a new observation, the target variable $Y$ can be inferred by sampling from the encoders $P_{\theta_k^*}(U_k|X_k)$ and the soft estimate provided by decoder $Q_{\phi_{\mc K}^*}(Y|U_1,\ldots, U_K)$.

%----------------------------------------
\begin{figure}[t!]
\vspace{-2mm}
\centering
\includegraphics[width=0.75\textwidth]{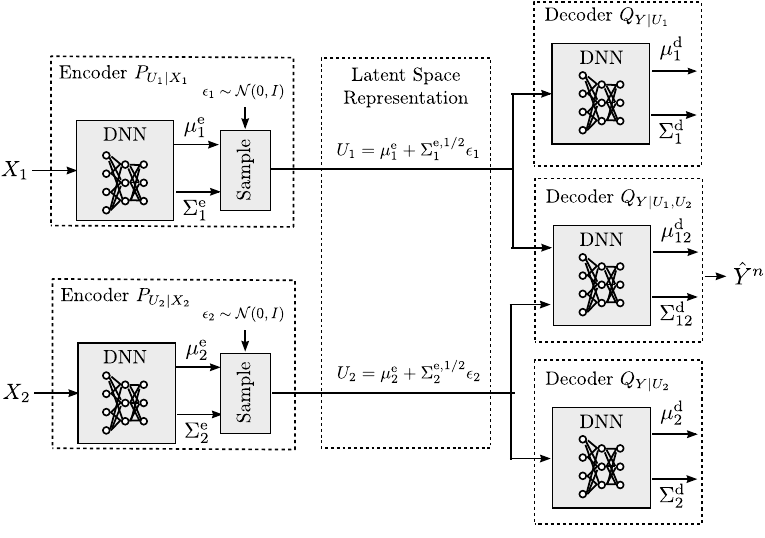}
\caption{ 
Distributed learning architecture for the minimization of the variational DIB cost in~\eqref{eq:VarEq} for $K=2$ using multivariate Gaussian distributions to parameterize the encoders, decoders and prior distributions. The decoders $Q_{Y|U_1}$ and $Q_{Y|U_2}$ influence in the regularization.}
\label{fig:Latent_Variables}
\vspace{-3mm}
\end{figure}
%----------------------------------------

\subsection{D-VIB for Regression and Classification }

The choice of parametric distributions  $P_{\theta_k}(u_k|x_k)$, $Q_{\psi_k}(y|u_k)$, $Q_{\psi_{\mc K}}(y|u_{\mc K})$ and $Q_{\varphi_k}(u_k)$ depends on the application. Nevertheless, the parametric families of distributions should be chosen to be expressive enough to approximate the optimal encoders minimizing~\eqref{eq:CostF} and the optimal conditional distributions~\eqref{eq:Qstark} and~\eqref{eq:Qstarall}, to minimize the gap between the lower bound given by the variational DIB cost~\eqref{eq:FunctionPQ} and the original DIB cost function~\eqref{eq:CostF}. 

For example, in classification problems, i.e., $\mc Y$ is finite, the decoder $Q_{\phi_{\mc K}}(Y|U_{\mc K})$ and the decoders $Q_{\phi_k}(Y|U_k)$, used for regularization can be general categorical distributions parameterized by a DNN with a softmax operation in the last layer, which outputs a vector of dimension $|\mc Y|$. If the target variable $Y$ is continuous, i.e., in regression problems, the decoders can be chosen to lie within the family of multivariate Gaussian distributions. The priors $Q_{\varphi_k}$ and the encoders $P_{\theta_k}(U_k|X_k)$, can be chosen as multivariate Gaussian, such that the divergence in the cost function is easy to compute \cite{KW13,AFDM:ICLR:2017}, or using more expressive parameterization  \cite{KSW16, PPM17}.  See Section~\ref{sec:ExperimentalResults} for some choices for different regression and classification.

%
%\begin{remark}
%It is shown in Theorem~\ref{th:GaussSumCap} in Section~\ref{ssec:GaussRegion} that if the data model is jointly vector Gaussian, the encoders, decoders and priors that minimize the variational D-IB cost are Gaussian in \eqref{eq:VarEq}  and, therefore, under Gaussian parametraztion,  we have $\mc{L}_s(\dv P) = \mc L^{\mathrm{VB}}_s(\boldsymbol\theta, \boldsymbol\phi, \boldsymbol \varphi)$. 
%\end{remark}
%

%As shown in \textbf{CARATHEODORY}, categorical latent spaces are optimal for discrete data models.

\subsection{D-VIB algorithm for Vector Gaussian Model }\label{ssec:GaussVDIB}
In the D-VIB in Algorithm~\ref{algo:VDIB}, the performance depends on the choice of the parametric family of distributions for the encoders, decoder and prior distributions. By Corollary~\ref{coro:GaussPmf}, when the underlying data model is multivariate vector Gaussian, the optimal distributions $\dv P$ and $\dv Q$ in~\eqref{eq:VarEq} lie within the family of multivariate Gaussian distributions. Motivated by this observation, we consider the following parameterization for $k\in \mc K$:
\begin{align}
P_{\theta_k}( \dv u_k|\dv x_k) &= \mc {N}(\dv u_k;{\boldsymbol{\mu}}_k^{\mathrm{e}}, {\dv \Sigma}_k^{\mathrm{e}})\label{eq:Gauss1}\\
Q_{\phi_{\mc K}}( \hat{\dv y}|\dv u_{\mc K}) &= \mc {N}(\hat{\dv y};{\boldsymbol{\mu}}_{\mc K}^{\mathrm{d}}, {\dv \Sigma}_{\mc K}^{\mathrm{d}})\\
Q_{\phi_k}( \hat{\dv y}|\dv u_{k}) &= \mc {N}(\hat{\dv y};{\boldsymbol{\mu}}_k^{\mathrm{d}}, {\dv \Sigma}_k^{\mathrm{d}})\quad k=1,2,\\
Q_{\varphi_k}(\dv u_k) &= \mathcal{N}(\dv 0, \dv I).\label{eq:Gauss5}
\end{align}
where ${\boldsymbol{\mu}}_k^{\mathrm{e}}, {\dv \Sigma}_k^{\mathrm{e}}$ are the output of a DNN $f_{\theta_k}$ with input $\dv X_k$ that encodes the input into a $n_{u_k}$-dimensional Gaussian distribution, ${\boldsymbol{\mu}}_{\mc K}^{\mathrm{d}}, {\dv \Sigma}_{\mc K}^{\mathrm{d}}$ are the outputs of a DNN $f_{\phi_{\mc K}}$ with inputs $\dv U_1,\ldots, \dv U_K$, sampled from $P_{\theta_k}( \dv u_k|\dv x_k)$  and ${\boldsymbol{\mu}}_k^{\mathrm{d}}, {\dv \Sigma}_k^{\mathrm{d}}$ are the output of a DNN $f_{\phi_{k}}$ with input $\dv U_k$, $k=1,\hdots, K$.

\begin{remark}
The above parametrization might result in some performance loss,  e.g., observe that $Q_{\varphi_k}(\dv u_k)$ may not be expressive enough to approximate the optimal distribution in \eqref{eq:Qstark}. Also, note that the distributions that optimize the variational DIB cost~\eqref{eq:VarEq} do not necessarily minimize the empirical DIB cost~\eqref{eq:Lhatoptimization}.  
\end{remark}

\section{Experimental Results}\label{sec:ExperimentalResults}
In this section, we evaluate the relevance-complexity tradeoffs achieved by the proposed algorithms BA-DIB and D-VIB in experiments with synthetic and real data. 
We also compare the resulting relevance-complexity pairs to the optimal relevance-complexity tradeoff and to an upper bound on the relevance-complexity, which we denote by 
%\begin{itemize}
 \textit{Centralized IB upper bound (C-IB)}. The C-IB bound is given by the  pairs $(\Delta,R_{\mathrm{sum}})$ achievable if $(X_1,\ldots, X_K)$ are encoded jointly at a single encoder with complexity  $R_{\mathrm{sum}}= R_1+\cdots+R_{K}$ and are given by the centralized IB problem~\cite{Tishby99theinformation}:
\begin{align}
\Delta_{\mathrm{cIB}}(R_{\mathrm{sum}}) = &\max_{P_{U|X_1,\ldots, X_K}}I(Y;U)\nonumber\\
&\text{s.t. }R_{\mathrm{sum}}\geq I(X_1,\ldots, X_K;U).\label{eq:IB_all}
\end{align}

%\item \textit{C-IB upper bound with $R_{\mathrm{sum}}\rightarrow \infty$}: The relevance-complexity pairs $(\Delta,R_{\mathrm{sum}})$ achievable under C-IB encoding in \eqref{eq:IB_all} when complexity is not limited, i.e.,  $R_{\mathrm{sum}}\rightarrow \infty$  is given by
%\begin{align}
%\Delta_{\mathrm{cIB}}(\infty) = I(Y;X_1,\ldots, X_K).
%\end{align}
%\end{itemize}

%Also, for comparison we consider the centralized version of the D-VIB algorithm from~\cite{AFDM:ICLR:2017},  which we denote by centralized variational information bottleneck (C-VIB) algorithm, in which the observations $(X_1,\ldots, X_K)$ are jointly encoded to a single latent space description from which the soft estimation of $Y$ is obtained.

Our implementation of the D-VIB in Algorithm \ref{algo:VDIB}  uses Tensorflow and  its implementation of Adam optimizer \cite{KDB14} over 150 epochs and minibatch size of 64. The learning rate  is computed as $0.001\cdot(0.5)^{\lfloor n_{\mathrm{epoch}}/30 \rfloor}$ at epoch $n_\text{epoch}$.

%%----------------------------------------
%\begin{figure}[t!]
%\centering
%\includegraphics[width=0.46\textwidth,trim=7 9 7 10,clip]{Figures/theory.pdf}
%\vspace{-4mm}
%\caption{Optimal relevance vs. sum-complexity tradeoff for vector Gaussian D-IB with $K=2$ encoders, source dimension $N=2$, and noisy observation dimension \mbox{$M_1 = M_2 = 3$, and performance of the BA-DIB and D-VIB algorithms.}} 
%\label{fig:theory}
%\end{figure}
%%----------------------------------------

%\begin{figure}[!t]
%\centering
%\includegraphics[width=0.465\textwidth,trim=7 9 7 10,clip]{Figures/mmse.pdf}
%\caption{Mean square error vs. sum-complexity tradeoff for vector Gaussian data model with $K=2$ encoders, $n_y=1$, $n_1 = n_2 = 3$, and achievable pairs with the BA-DIB and D-VIB algorithms for $n=30.000$.
%} 
%\label{fig:theory_mmse}
%\end{figure} 

\begin{figure}[!t]
\begin{minipage}[t]{0.465\linewidth}
\centering
\includegraphics[width=0.975\textwidth,trim=7 9 7 10,clip]{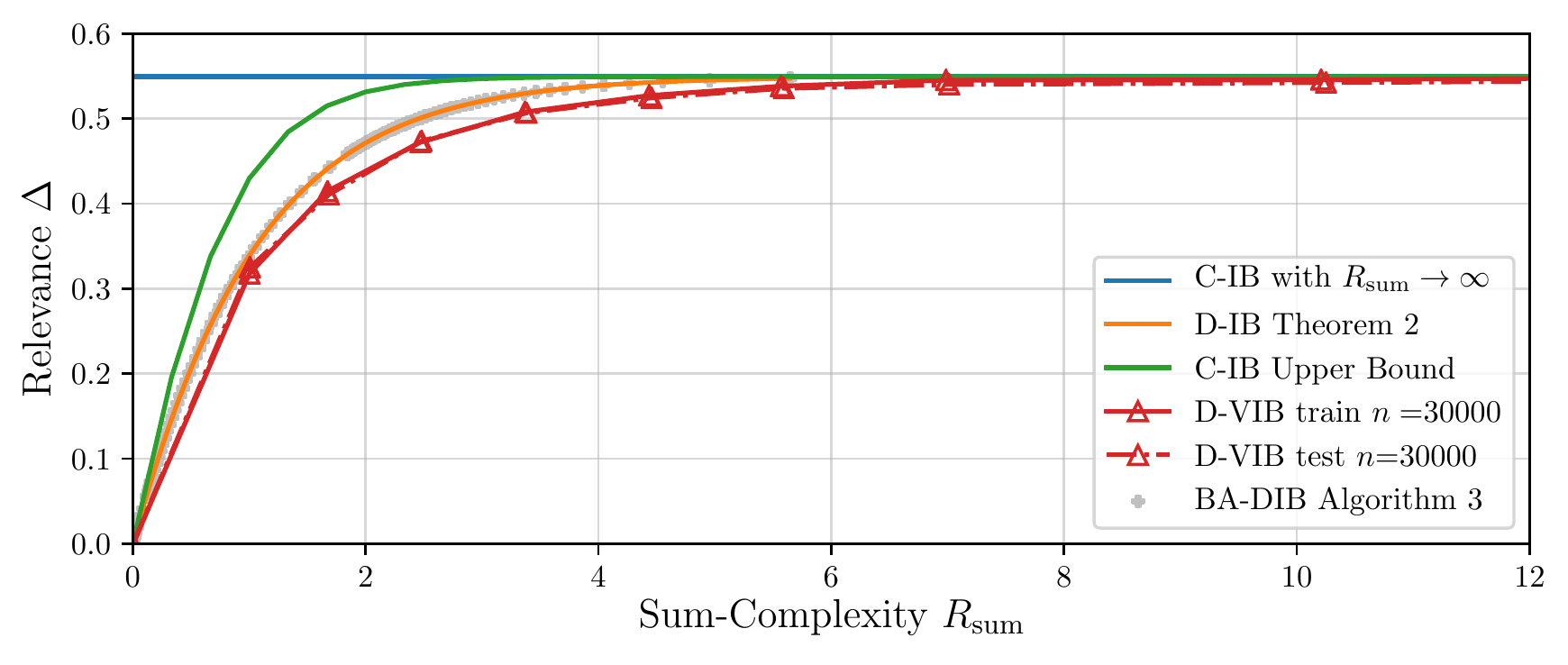}
\caption{Relevance vs. sum-complexity tradeoff for vector Gaussian data model with $K=2$ encoders, $n_y=1$, $n_1 = n_2 = 3$, and achievable pairs with the BA-DIB and D-VIB algorithms for $n=30.000$.
} 
\label{fig:theory}
 \end{minipage}
 \hspace{0.1cm}
 \begin{minipage}[t]{0.465\linewidth}
\centering
\includegraphics[width=0.975\textwidth,trim=7 9 7 10,clip]{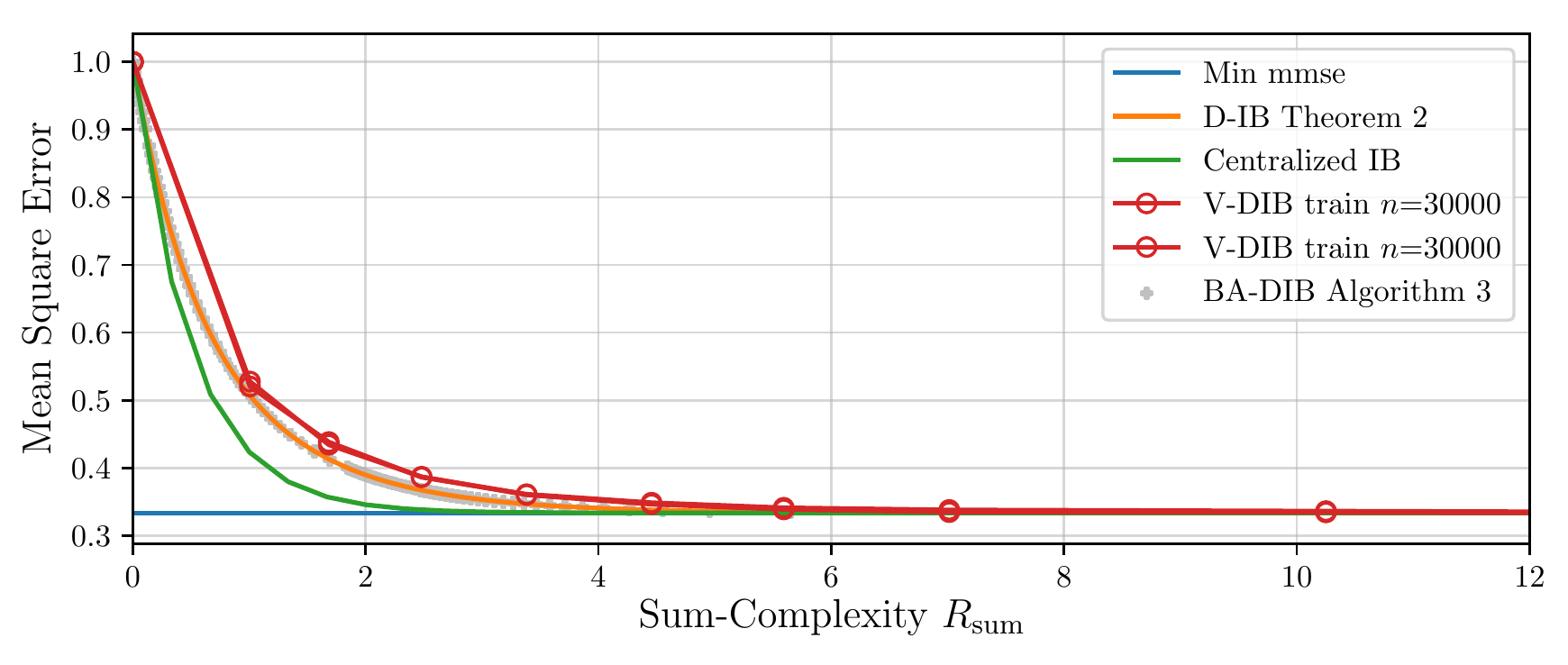}
\caption{Mean square error vs. sum-complexity tradeoff for vector Gaussian data model with $K=2$ encoders, $n_y=1$, $n_1 = n_2 = 3$, and achievable pairs with the BA-DIB and D-VIB algorithms for $n=30.000$.
} 
\label{fig:theory_mmse}
\end{minipage}
\end{figure}

\subsection{Regression for Vector Gaussian Data Model}
In this section, we consider a real valued vector Gaussian data model as in Section~\ref{sec:Gauss} with $K=2$ encoders observing a noisy version of an $n_y$-dimensional Gaussian vector $\dv Y\sim \mc{N}(\dv y; \dv 0, \mathbf{I})$, as $\mathbf{X}_k = \mathbf{H}_{k}\mathbf{Y}+\mathbf{N}_k$, where  $\mathbf{H}_{k}\in \mathds{R}^{n_k\times n_y}$ and the noise is distributed as $\mathbf{N}_k\sim \mc{N}(\dv 0,\mathbf{I})$, $k  = 1,2$. 

The optimal complexity-relevance tradeoff for this model is characterized as in Theorem~\ref{th:GaussSumCap} and can be computed by solving the convex problem~\eqref{th:GaussSumCap}. The C-IB upper bound in~\eqref{eq:IB_all} is an instance the IB problem for vector Gaussian sources in~\cite{journals/jmlr/ChechikGTW05} and it can be computed analytically. 
% When complexity is not limited, i.e., $R_{\mathrm{sum}}\rightarrow \infty$, we have,
%\begin{align}
%\Delta_{cIB}(\infty) = I(\dv Y;\dv X_1,\dv X_2)=  \frac{1}{2}\log\det\left([\dv H_1^T, \dv H_2^T]^T\boldsymbol\Sigma_{\dv y}[\dv H_1,\dv H_2]+\mathbf{I}\right).
%\end{align}

Next, we consider the numerical evaluation of the proposed algorithms BA-DIB and D-VIB for regression of the Gaussian target variable $\dv Y$ trained using a dataset of $n$ i.i.d. samples $ \{( \dv X_{1,i}, \dv X_{2,i}, \dv Y_i)\}_{i=1}^n$ form the described vector Gaussian data model. For BA-DIB, we assume a jointly Gaussian distribution for the data and empirically estimate its joint mean and covariance. Then, we apply Algorithm~\ref{algo:BA_Gauss} to compute the relevance-complexity pairs and the corresponding estimators for different values of $s$. For D-VIB, we do not make any assumption on the data model and we apply Algorithm~\ref{algo:VDIB} to train the DNNs determining the encoders and decoders for different values of $s$. We use the multivariate Gaussian parameterization in \eqref{eq:Gauss1}-\eqref{eq:Gauss5} for the DNNs architecture shown in Table~\ref{table:Architecture_Gauss}. Specifically, Encoder $k$, $k=1,2$, consists of three dense layers of 512 neurons each followed by rectified linear unit (ReLu) activations. The output of encoder $k$ is processed by a dense layer without nonlinear activation to generate ${\boldsymbol \mu}_k^{\mathrm{e}}$ and ${\dv\Sigma}_k^{\mathrm{e}}$ of size $512$ and $512\times 512$, respectively. Each decoder consists of two dense layers of 512 neurons with ReLu activations. The output of decoder $1$, $2$ and $12$ is processed, each, by a fully connected layer without activation to generate ${\boldsymbol \mu}_k^{\mathrm{d}}$ and ${\dv\Sigma}_k^{\mathrm{d}}$ and ${\boldsymbol \mu}_{12}^{\mathrm{d}}$ and ${\dv\Sigma}_{12}^{\mathrm{d}}$, of size $2$ and $2\times 2$.

Figure \ref{fig:theory} shows the optimal relevance-complexity region of tuples $(\Delta,R_{\mathrm{sum}})$ obtained from Theorem~\ref{th:GaussSumCap} for a vector Gaussian  model with $K=2$ encoders, target variable dimension $n_y=1$, and observations dimension $n_1 = n_2 = 3$, as well as the C-IB bounds $\Delta_{\mathrm{cIB}}(R_{\mathrm{sum}})$ and $\Delta_{\mathrm{cIB}}(\infty)$. A dataset of $40.000$ i.i.d. samples is available which is split into a training set of $n=30.000$ samples and a test set of $10.000$ samples used to evaluate the results from the trained estimators. We show the tuples $(\Delta,R_{\mathrm{sum}})$  resulting from the application of the BA-DIB in Algorithm~\ref{algo:BA_Gauss} for different values of $s\in (0,10]$, and we observe that they lie on the optimal relevance-complexity curve obtained from Theorem~\ref{th:GaussSumCap}. The relevance-complexity pairs resulting form the application of the D-VIB algorithm for different values of $s$ in the range $(0,10]$ calculated as in Proposition~\ref{prop:param} are also shown. Figure \ref{fig:theory_mmse} depicts the mean squared error (MSE) between the original vector $\dv Y$ and the estimation given by the application of the BA-DIB and the D-VIB algorithm as in Figure~\ref{fig:theory} .

\begin{table}[t!]
\begin{minipage}[t]{0.5\linewidth}
\centering 
\includegraphics[width=0.975\textwidth,trim=7 9 7 10,clip]{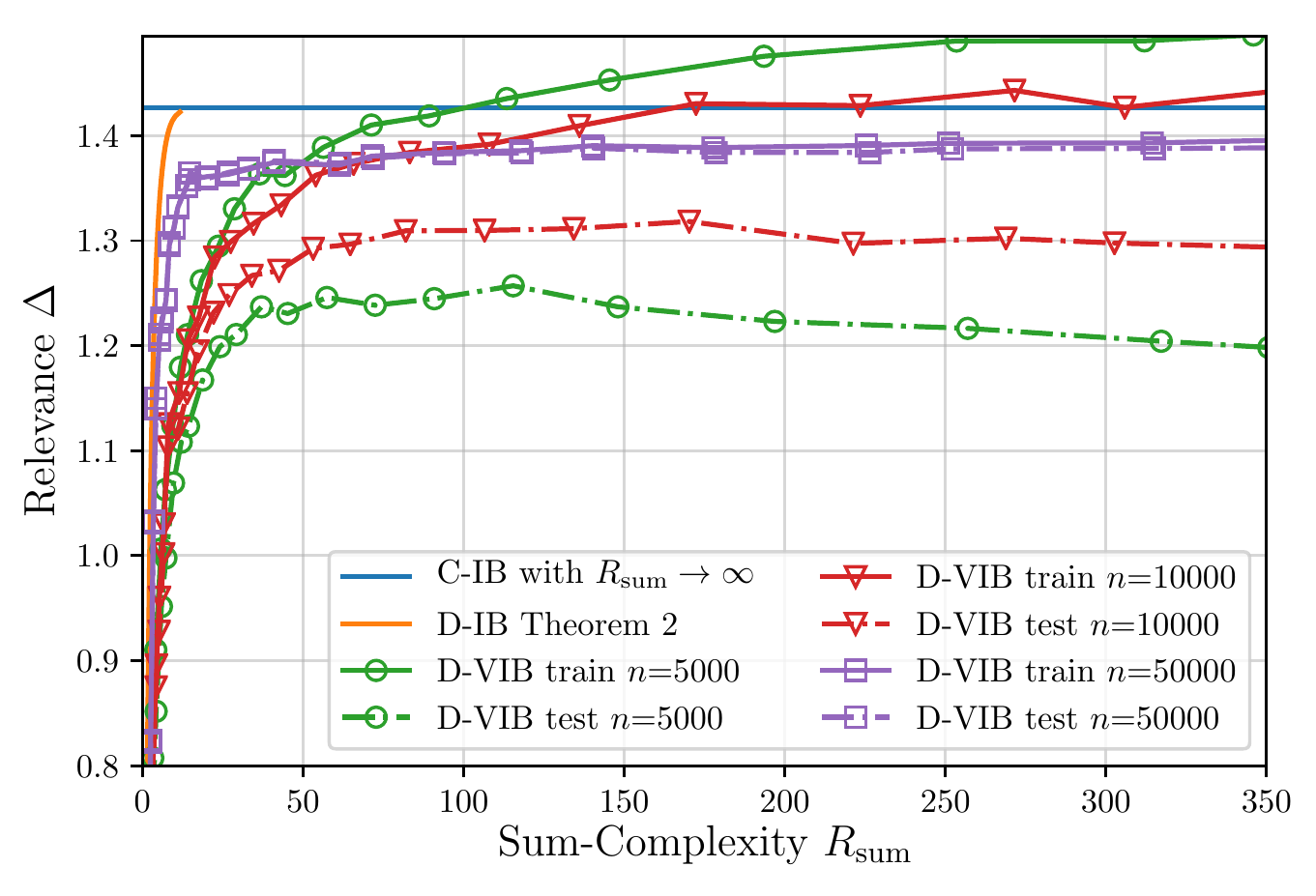}
\captionof{figure}{
Relevance vs. sum-complexity tradeoff for vector Gaussian data model with $K=2$ encoders, $n_y=2$, $n_1 = n_2 = 3$, and achievable pairs with D-VIB for training dataset $n=\{5.000, 10.000, 50.000\}$.} 
\vspace{-5mm}
\label{fig:GaussDIB_TrainTest}
 \end{minipage}
{\footnotesize
\begin{minipage}[t]{0.475\linewidth} 
\vspace{-45mm}
 \caption{DNN architecture for Figure~\ref{fig:theory} and Figure~\ref{fig:GaussDIB_TrainTest}}
\centering
%\begin{table*}\centering
%\ra{1.3}
\begin{tabular}{@{}lc@{}}\toprule
&DNN Layers\\
 \midrule
Encoder $k$ & dense [512]-ReLu\\
& dense [512]-ReLu\\
& dense [512]-ReLu\\
Lat. space $k$  & dense [256]-ReLu\\
Decoder $12$ & dense [256]-ReLu\\
Decoder $k$ & dense [256]-ReLu \\
\bottomrule
\label{table:Architecture_Gauss}
\end{tabular}
\end{minipage}}
\end{table}

Figure~\ref{fig:GaussDIB_TrainTest} shows the effect of the size of the training set on generalization of the DNNs trained with the D-VIB in Algorithm~\ref{algo:VDIB} for  a vector Gaussian  model with $K=2$ encoders, target variable dimension $n_y=2$, and observations dimension $n_1 = n_2 = 3$, and the same DNN architecture in Table~\ref{table:Architecture_Gauss}. The achievable relevance-complexity pairs are shown after training the DNNs with training set lengths of $n=\{5.000, 10.000, 50.000\}$ as well as the corresponding achievable pairs for the test data, and the optimal pairs from Theorem~\ref{th:GaussSumCap}.
It is observed that for the training data, the larger the complexity, the higher the achievable relevance. Note that for small and intermediate datasets, e.g., with $n=\{5.000,10.000\}$, there is overfitting of the DNN to the dataset, and the resulting relevance is higher than that allowed by the optimal tradeoff. However, overfitting results in performance loss for test data, which is more pronounced for larger complexity values. Note that the complexity constraint acts as a regularizer which improves generalization of the estimators and that the largest relevance during test is achieved for small complexity values. As the dataset size increases, the difference between resulting relevance-complexity pairs for the test and train datasets is reduced and the better the DNN generalizes.  Note the gap to the optimal tradeoff, which could be reduced with more expressive parameterizations than that in~\eqref{eq:Gauss1}-\eqref{eq:Gauss5}.

%Figure~\ref{fig:centralized-vs-decentralized} show the rate pairs resulting from the application of the centralized V-IB and the descentralized D-VIB algorithm with respect to the optimal centralized and distributed performances. We observe that from the same dataset, the centralized algorithm achieves larger relevant information than D-IB and performs close to the optimal tradeoff, while for the test phase, the performance of D-IB achieves higher relevance in certain cases. This might be causes due to a better regularization effect in the variational DIB cost~\eqref{eq:FunctionPQ}

\vspace{-2mm}
\subsection{Classification on the multi-view MNIST Dataset}\label{secc:MultiviewMnist}

xIn this section, we consider the evaluation of the proposed algorithms for classification on a two-view version of the the MNIST dataset, consisting of $70.000$ labeled images of handwritten digits between $\{0,\ldots, 9\}$ observed under two different views, from which the corresponding label $Y\in \{0,\ldots, 9\}$ has to be inferred. Each MNIST image consists of $28\times 28$ grayscale pixels.  In this experiment the two views are generated as follows.
View~1 is generated by occluding the image with square of $25 \times 25$ pixels, randomly rotated for each image with angles uniformly selected in the range $[-\pi/4,\pi/4]$. To generate view~2, for each image in view~1, we randomly select an image of the same digit $(0-9)$ from the original MNIST dataset and add independent random noise uniformly sampled from $[0,3]$ to each pixel, and the pixel values are truncated to $[0,1]$.  Encoder $1$ observes view $1$ and Encoder $2$ observes view $2$, and treat each view as normalized vectors $\dv X_k\in [0,1]^{784}$, $k=1,2$. We randomly split the $70.000$ two-view samples $\{\dv X_{1,i}, \dv X_{2,i},  Y_i \}$ into training and test sets of length $n$ and $70.000-n$, respectively. A subset of the two-view MNIST dataset is shown in Figure~\ref{fig:Multi_view_Mnist}. To asses the difficulty of estimating the digit from each of the views independently, we consider a standard convolutional neural network (CNN) architecture with dropout which achieves an accuracy of $99.8\%$ in the noiseless MNIST dataset. Then, we independently retrain the same CNN architecture for each encoder input, achieving an accuracy of $92.3\%$ for view~1 and $79.68\%$ for view 2, as shown in Table~\ref{table:Accuracy}. That is, view 1 is less noisy.

%%----------------------------------------
%\begin{figure}[t!]
%\centering
%\includegraphics[width=0.46\textwidth]{Figures/Multi_view_Mnist.pdf}
%\caption{Generated Two-view MNIST dataset from the original MNIST handwritten dataset.}
%\label{fig:Multi_view_Mnist}
%\end{figure}
%%----------------------------------------

\begin{table}[t!]
\begin{minipage}[t]{0.5\linewidth}
\centering
\includegraphics[width=0.95\textwidth]{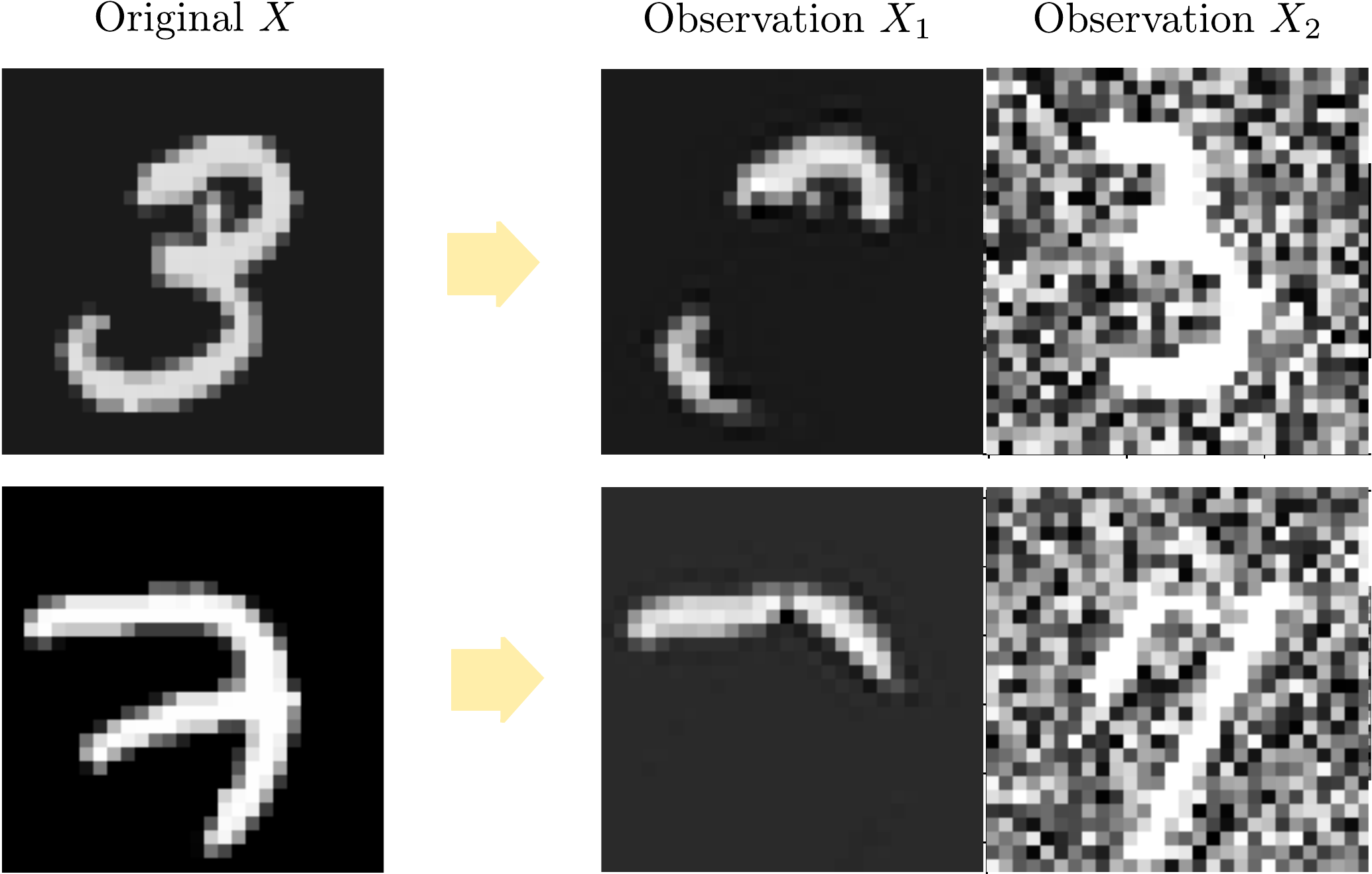}
\captionof{figure}{Two-view handwritten MNIST dataset.}
\label{fig:Multi_view_Mnist}
 \end{minipage}
~
{\footnotesize
\begin{minipage}[t]{0.475\linewidth} 
\vspace{-50mm}
 \caption{DNN architecture for Figure~\ref{fig:MNIST_VDIB} and Figure\ref{fig:MNISTDIB_acc_TrainTest}}
\centering
%\begin{table*}\centering
%\ra{1.3}
\begin{tabular}{@{}lc@{}}\toprule
&DNN Layers\\
 \midrule
Encoder $k$ & conv. ker. [5,5,32]-ReLu\\
 & maxpool [2,2,2]\\
& conv. ker. [5,5,64]-ReLu\\
 & maxpool [2,2,2]\\
 & dense [1024]-ReLu\\
 & dropout 0.4\\
& dense [256]-ReLu \\
Latent space $k$ & dense [256]-ReLu\\
Decoder $12$ & dense [256]-ReLu\\
Decoder $k$ & dense [256]-ReLu \\
\bottomrule
\label{table:Architecture_MNIST}
\end{tabular}
\end{minipage}}
\end{table}

%\begin{figure}[!t]
%\centering 
%\includegraphics[width=0.25\textwidth]{Figures/Multi_view_Mnist2.pdf}
%\captionof{figure}{Two-view handwritten MNIST dataset.}
%\label{fig:Multi_view_Mnist}
%\vspace{-2mm}
%\end{figure} 

%\begin{table*}\centering
%\vspace{-7mm}
%\ra{1.3}
%\begin{tabular}{@{}cc@{}}\toprule
%&NN Layers\\
% \midrule
%Encoder $k$ & [784,1024, 1024]\\
%Decoder $12$ & [256]\\
%Decoder $k$ & [256] \\
%\bottomrule
%\end{tabular}
%\caption{Caption}
%\end{table*}

We consider the application of the D-VIB algorithm to this model with the CNN architecture in Table~\ref{table:Architecture_MNIST}, in which Encoder $k$, $k=1,2$ is parameterized by a $n_{u_k} = 256$ dimensional multivariate Gaussian distribution $\mc{N}({\boldsymbol \mu}_k^{\mathrm{e}},{\boldsymbol 
\Sigma}_k^{\mathrm{e}})$ determined by the output of a DNN $f_{\theta_k}$ consisting of the concatenation of convolutional, dense and maxpool layers with ReLu activations and dropout. The output of the last layer is followed by a dense layer without activation that generate ${\boldsymbol \mu}_k^{\mathrm{e}}$ and ${\boldsymbol 
\Sigma}_k^{\mathrm{e}}$. The prior is chosen as $Q_{\varphi_k}(\dv u) = \mathcal{N}(\dv 0, \dv I)$. Each decoder takes the samples from $P_{\theta_k}(U_k|X_k)$ and processes its inputs with a dense layer DNN ($f_{\phi_{\mc K}}$ and $f_{\phi_{k}}$) each with 256 neurons and ReLu activation, which  outputs a vector $\hat{\dv y}_i$ of size $|\mc Y| = 10$ normalized with a softmax, corresponding to a distribution over the one-hot encoding of the digit labels $\{0,\ldots, 9\}$ from the $K$ observations, i.e., we have:
\begin{align}
Q_{\phi_k}( \hat{\dv y}_k|\dv u_{k}) &= \text{Softmax}(f_{\phi_k}(U_k)), \quad k=1,2, \text{ and }\\
Q_{\phi_{\mc K}}( \hat{\dv y}|\dv u_{\mc K}) &= \text{Softmax}(f_{\phi_{\mc K}}(U_1, U_2))), 
\end{align}
where $\text{Softmax}(\dv p)$ for $\dv p\in \mathds{R}^d$ is a vector with $i$-th entry $[\text{Softmax}(\dv p)]_i = \exp(p_i)/\sum_{j=1}^{d}\exp(p_j)$, $i=1,\ldots, d$.

For this parameterization, the log-loss terms in the empirical DIB cost~\eqref{eq:EmpiricalL}  reduce to the cross-entropy and the KL divergence terms, can be computed as in~\eqref{eq:DistGaussi}.
%
%\begin{figure}[!t]
%\begin{minipage}[t]{0.5\linewidth}
%\centering 
%\includegraphics[width=0.95\textwidth,trim=7 9 7 10,clip]{Figures/MNISTDIB_muTrainTest.pdf}
%\caption{Relevance vs. sum-complexity tradeoff for the two-view MNIST dataset with $K=2$ encoders, with the D-VIB algorithm for training dataset $n=50.000$ and $s\in[10^{-7},1]$.}  
%\label{fig:MNIST_VDIB}
% \end{minipage}
%\hspace{0.1cm}
%\begin{minipage}[t]{0.5\linewidth} 
%\centering 
%\includegraphics[width=0.95\textwidth,trim=7 9 7 10,clip]{Figures/MNISTDIB_acc_TrainTest.pdf}
%\caption{Train and test accuracy for the two-view MNIST dataset with $K=2$ encoders, from the estimators resulting with the D-VIB algorithm and the D-VIB av for training dataset $n=50.000$ and  $s\in[10^{-7},1]$.} 
%\label{fig:MNISTDIB_acc_TrainTest}
%\end{minipage}
%\vspace{-6mm}        
%\end{figure} 

\begin{figure}[!t]
\begin{minipage}[t]{0.465\linewidth}
%%%
\centering 
\includegraphics[width=0.975\textwidth,trim=7 10 7 10,clip]{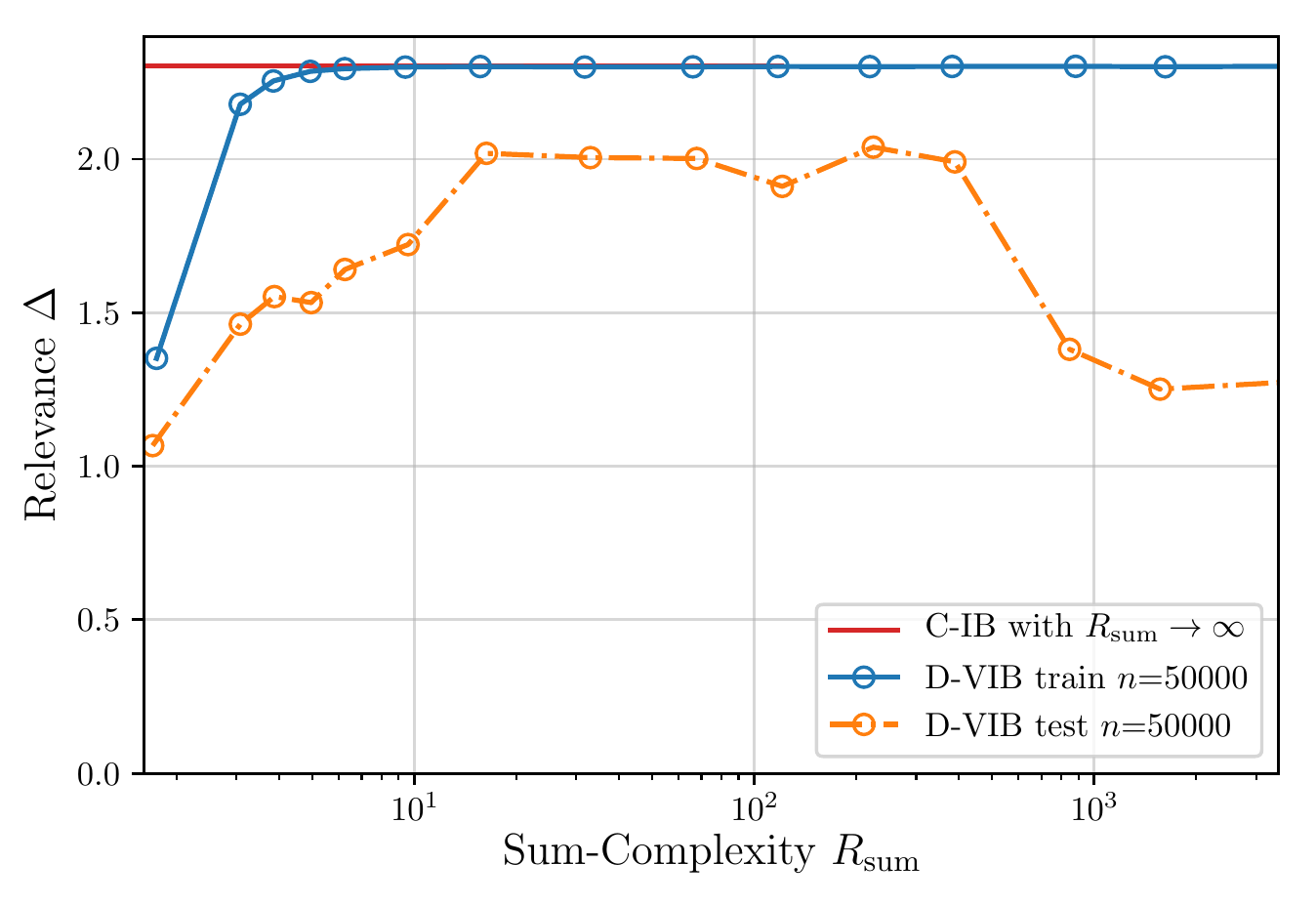}
\caption{Relevance vs. sum-complexity tradeoff for the two-view MNIST dataset with $K=2$ encoders, with the D-VIB algorithm for training dataset $n=50.000$ and $s\in[10^{-10},1]$.}  
\label{fig:MNIST_VDIB}
%\hspace{0.1cm}
%%%%
 \end{minipage}
 \hspace{0.1cm}
 \begin{minipage}[t]{0.465\linewidth}
%%%
\centering 
\includegraphics[width=0.975\textwidth,trim=7 7 7 10,clip]{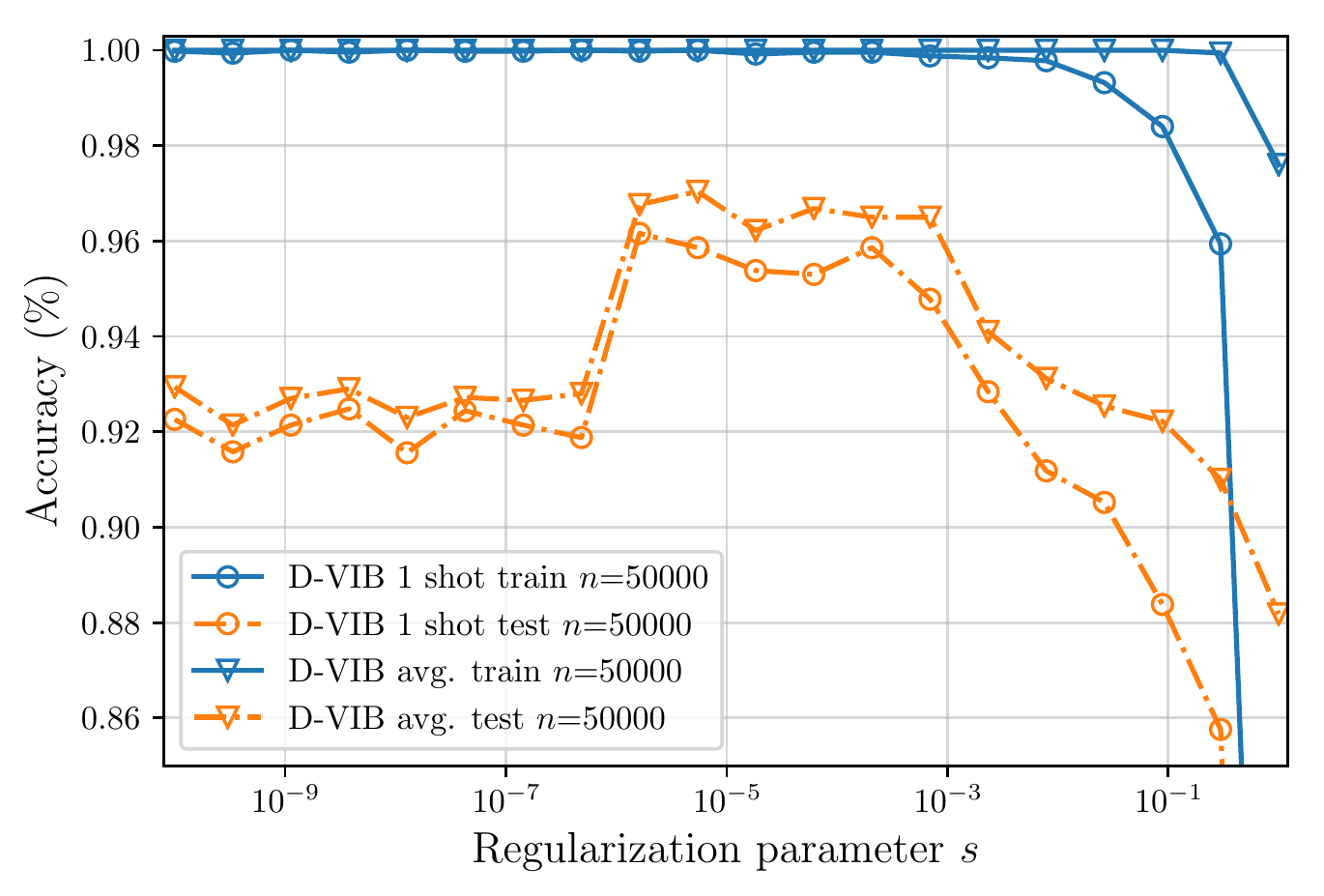}
\caption{Train and test accuracy for the two-view MNIST dataset with $K=2$ encoders, from the estimators resulting with the D-VIB algorithm and the D-VIB avg. for  $n=50.000$ and  $s\in[10^{-10},1]$.} 
\label{fig:MNISTDIB_acc_TrainTest}  
%%%
\end{minipage}
\end{figure}

\begin{table}[t!]
\begin{minipage}[t]{0.5\linewidth}
%%%
\includegraphics[width=0.975\textwidth,trim=7 7 7 10,clip]{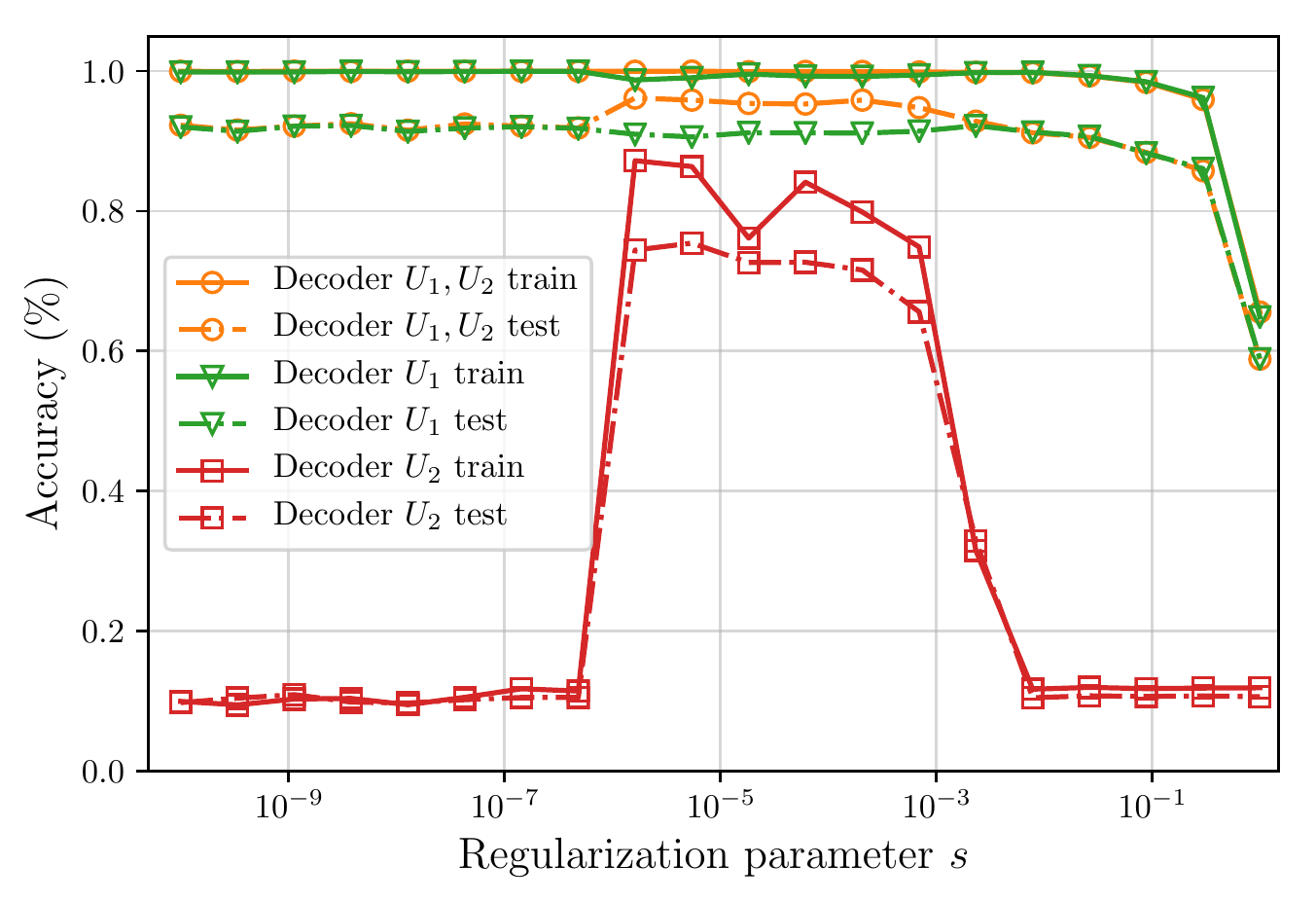}
\captionof{figure}{Train and test accuracy for the two-view MNIST dataset with $K=2$ encoders, from the estimators using both $(U_1,U_2)$, $Q_{Y|U_1,U_2}$, and only $U_1$ or $U_2$, i.e., $Q_{Y|U_1}$ and $Q_{Y|U_2}$,  from the application of the D-VIB algorithm for  $n=50.000$ and  $s\in[10^{-10},1]$.} \label{fig:MNISTDIB_acc_12_1_2_TrainTest}  
%%%%
 \end{minipage}
~
{\footnotesize
\begin{minipage}[t]{0.475\linewidth} 
\vspace{-45mm}
\centering
{\footnotesize
 \caption{Accuracy for different algorithms with CNN architectures.}
%\begin{table*}\centering
%\ra{1.3}
\begin{tabular}{@{}lcc@{}}\toprule
&Accuracy (\%)\\
& 1 shot & avg.\\
 \midrule
D-VIB& \textbf{96.16}& \textbf{97.24}\\
D-VIB-noReg & 96.04& 96.72 \\
C-VIB& 96.01& 96.68\\
Deterministic &93.18&93.18\\
Independent& 92.1 / 79.68 & 93.1 / 82.01\\
\bottomrule
\end{tabular}
\label{table:Accuracy}}
\end{minipage}}
\end{table}

%\begin{figure}[!t]
%\centering 
%\includegraphics[width=0.45\textwidth,trim=7 10 7 10,clip]{Figures/MNISTDIB_muTrainTest.pdf}
%\caption{Relevance vs. sum-complexity tradeoff for the two-view MNIST dataset with $K=2$ encoders, with the D-VIB algorithm for training dataset $n=50.000$ and $s\in[10^{-10},1]$.}  
%\label{fig:MNIST_VDIB}
%%\hspace{0.1cm}
%\end{figure} 

%\begin{table}[t!]
%\centering
%{\footnotesize
% \caption{Accuracy for different algorithms with CNN architectures.}
%%\begin{table*}\centering
%%\ra{1.3}
%\begin{tabular}{@{}lcc@{}}\toprule
%&Accuracy (\%)\\
%& 1 shot & avg.\\
% \midrule
%D-VIB& \textbf{96.16}& \textbf{97.24}\\
%D-VIB-noReg & 96.04& 96.72 \\
%C-VIB& 96.01& 96.68\\
%Deterministic &93.18&93.18\\
%Independent& 92.1 / 79.68 & 93.1 / 82.01\\
%\bottomrule
%\end{tabular}
%\label{table:Accuracy}}
%\end{table}

Figure~\ref{fig:MNIST_VDIB} shows the resulting relevance-complexity pairs from the application of the D-VIB Algorithm~\ref{algo:VDIB} on the two-view MNIST dataset for training dataset of size $n=50.000$ and 15 different $s$ regularization values in the range $[10^{-10}, 1]$ to train different estimators on the relevance-complexity plane. The C-IB limit is shown for $R_{\mathrm{sum}}\rightarrow \infty$ assuming that zero classification error is possible form the data, given by $\Delta_{\mathrm{cIB}}(R_{\mathrm{sum}}) = \log 10$. It can be observed that during the training phase, the higher the sum-complexity, the higher the achieved relevance, and that the resulting relevant-complexity pairs  perform very close to the theoretical limit.
On the contrary, during the test phase, while low sum-complexity results in low achievable relevance, and it increases for intermediate sum-complexity values, the achievable relevance decreases for large values of sum-complexity. Thus, the effect of the regularization due to the complexity constraint results in higher generalization. 

 After training, the trained CNNs for different values of $s$ can be used for classification. In particular we estimate the labels from the maximum of the estimated conditional distribution  $\hat{\dv y}_i$ for the given observation. We study the case in which a single Monte Carlo sample of $U_k$ is used for prediction, denoted by D-VIB 1 shot, and the case, which we denote by D-VIB avg., in which $U_k$ are sampled $M$ times to generate $M$ estimations $\hat{\dv y}_{i,m}$, which are averaged to estimate the conditional probability $\hat{\dv y}_{\mathrm{av}} =  \frac{1}{M}\sum_{m=1}^{M}\hat{\dv y}_{i,m}$, from which the label can be inferred. 

%\begin{figure}[!t]
%\centering 
%\includegraphics[width=0.465\textwidth,trim=7 7 7 10,clip]{Figures/MNISTDIB_acc_TrainTest.pdf}
%\caption{Train and test accuracy for the two-view MNIST dataset with $K=2$ encoders, from the estimators resulting with the D-VIB algorithm and the D-VIB avg. for  $n=50.000$ and  $s\in[10^{-10},1]$.} 
%\label{fig:MNISTDIB_acc_TrainTest}     
%\end{figure} 

%\begin{figure}[!t]
%\centering 
%\includegraphics[width=0.45\textwidth,trim=7 7 7 10,clip]{Figures/MNISTDIB_acc_12_1_2_TrainTest.pdf}
%\caption{Train and test accuracy for the two-view MNIST dataset with $K=2$ encoders, from the estimators using both $(U_1,U_2)$, $Q_{Y|U_1,U_2}$, and only $U_1$ or $U_2$, i.e., $Q_{Y|U_1}$ and $Q_{Y|U_2}$,  from the application of the D-VIB algorithm for  $n=50.000$ and  $s\in[10^{-10},1]$.} \label{fig:MNISTDIB_acc_12_1_2_TrainTest}    
%\vspace{-4mm} 
%\end{figure} 

Figure~\ref{fig:MNISTDIB_acc_TrainTest} shows the resulting accuracy achievable with D-VIB 1 shot and D-VIB avg. using each estimator obtained in Figure~\ref{fig:MNIST_VDIB} with respect to the regularization parameter $s$. 
It can be observed that higher accuracy for both methods is obtained at the intermediate regularization parameter values, $s\simeq 10^{-6}$,  for which relevance is maximized in Figure~\ref{fig:MNIST_VDIB}. Indeed, for any estimator $Q_{Y|U_{\mc K}}$, the average logarithmic-loss provides an upper bound on the classification error, since by the application of Jensen's inequality, we have 
\begin{align}
P_{\mathrm{error}}(Q_{Y|U_{\mc K}}) &:= 1 - \mathrm{E}_{P_{X_{\mc K},Y}}[Q_{Y|U}]\\
& \leq 1 - \exp\left(- \mathrm{E}_{P_{X_{\mc K},Y}}[-\log Q_{Y|U_{\mc K}})]\right),
\end{align}
which justifies the logarithmic loss as a useful surrogate of the probability of error. Note that the largest gains from the averaging with D-VIB avg. occur for the values of $s$ for which the relevance is maximized for the test data.

{During the optimization, indirectly, the D-VIB algorithm allocates the complexity for each description $U_k$ of the observation $X_k$. Figure~\ref{fig:MNISTDIB_acc_12_1_2_TrainTest} shows the resulting accuracy achieved  with D-VIB 1 shot to estimate $Y$ with the main estimator $Q_{Y|U_1,U_2}$  obtained in Figure~\ref{fig:Latent_Variables}, as well as the accuracy, achieved by using the regularizing decoders $Q_{Y|U_k}$, $k=1,2$ to estimate $Y$, with respect to the regularization parameter $s$. In general, the description $U_1$ from view $X_1$ (which is less noisy), carries most of the information, while for $10^{-6}\lesssim s\lesssim 10^{-3}$, both descriptions $U_1$ and $U_2$ capture relevant information from the views. In this regime, the combination of both views result in an increase of the overall performance for $Q_{Y|U_1,U_2}$.  }

In order to asses the advantages of the D-VIB trained with Algorithm~\ref{algo:VDIB}, and the relevance of the logarithmic loss terms in the regularization in the variational DIB cost in~\eqref{eq:FunctionPQ}, we compare the best accuracy achievable by D-VIB in Figure~\ref{fig:MNISTDIB_acc_TrainTest}, i.e., selecting the estimator corresponding to $s$ for which the accuracy is maximized, with three more estimators for the setup in Figure~\ref{fig:Schm} using CNNs. First, we consider the centralized VIB algorithm, e.g., as in \cite{AFDM:ICLR:2017}, denoted by C-VIB, in which all observations are processed jointly. C-VIB can be trained with the DIB cost in~\eqref{eq:EmpiricalL} with $K=1$ and letting $X=(X_1,X_2)$. We also consider the training of an estimator without the regularizing decoders $Q_{Y|U_k}$ $k=1,2$ and by maximizing the empirical DIB cost in~\eqref{eq:EmpiricalL} with only the divergence terms in the regularizer using Algorithm~\ref{algo:VDIB}. The latter corresponds to the naive direct extension of the ELBO bound and the C-VIB to the distributed case, and we denote this scheme by D-VIB-noReg. 
In particular, for C-VIB, we assume that a single encoder maps the two views $\dv X_1,\dv X_2$ to a Gaussian latent of dimension $256$ using a CNN network of similar architecture as the encoders in Table \ref{table:Architecture_MNIST}, and uses a singe decoder with $256$ activations.
Finally, we consider a standard deterministic CNN with dropout regularization, using an architecture as that of the D-VIB-noReg, i.e., without regularizing decoders, and with a deterministic mapping for the latent space layer.
 For estimation we also consider the single shot and averaging of multiple estimations of the estimated conditional distribution as in D-VIB avg. Table~\ref{table:Accuracy} shows the advantages of D-VIB over other  approaches. The advantage of D-VIB over C-VIB is explained by the advantage of training the latent space embedding for each observation separately, which allows to adjust better the CNN encoding-decoding parameters to the statistics of each observation, justifying the use of D-VIB for multi-view learning even if the data is available in a centralized manner.

%\begin{table}[t!]
%
%\begin{minipage}[t]{0.35\linewidth} 
% \vspace{-45mm}
% \caption{Accuracy for different architectures}
%\centering
%%\begin{table*}\centering
%%\ra{1.3}
%\begin{tabular}{@{}lcc@{}}\toprule
%&Accuracy (\%)\\
%& 1 shot & avg.\\
% \midrule
%D-VIB& \textbf{98.07}& \textbf{98.47}\\
%D-VIB-noReg & 98.01& 98.31 \\
%C-VIB& 97.52& 97.62\\
%\bottomrule
%\end{tabular}
%\vspace{10mm}
%\label{table:Accuracy}
%%\end{table*}
%\end{minipage}
%~
%\begin{minipage}[t]{0.65\linewidth}
%\centering
%\includegraphics[width=0.995\textwidth]{Figures/MNIST_IB_TrainTest_traj.pdf}
%\captionof{figure}{Evalution of the relevance-complexity pairs during training with D-VIB for different values of the regularization parameter $s$ over the relevance-complexity plane.}
%\label{fig:MNIST_IB_TrainTest_traj}
% \end{minipage}
% \vspace{-2mm}
%\end{table}

\begin{figure}[!t]
\centering
\includegraphics[width=0.6\textwidth]{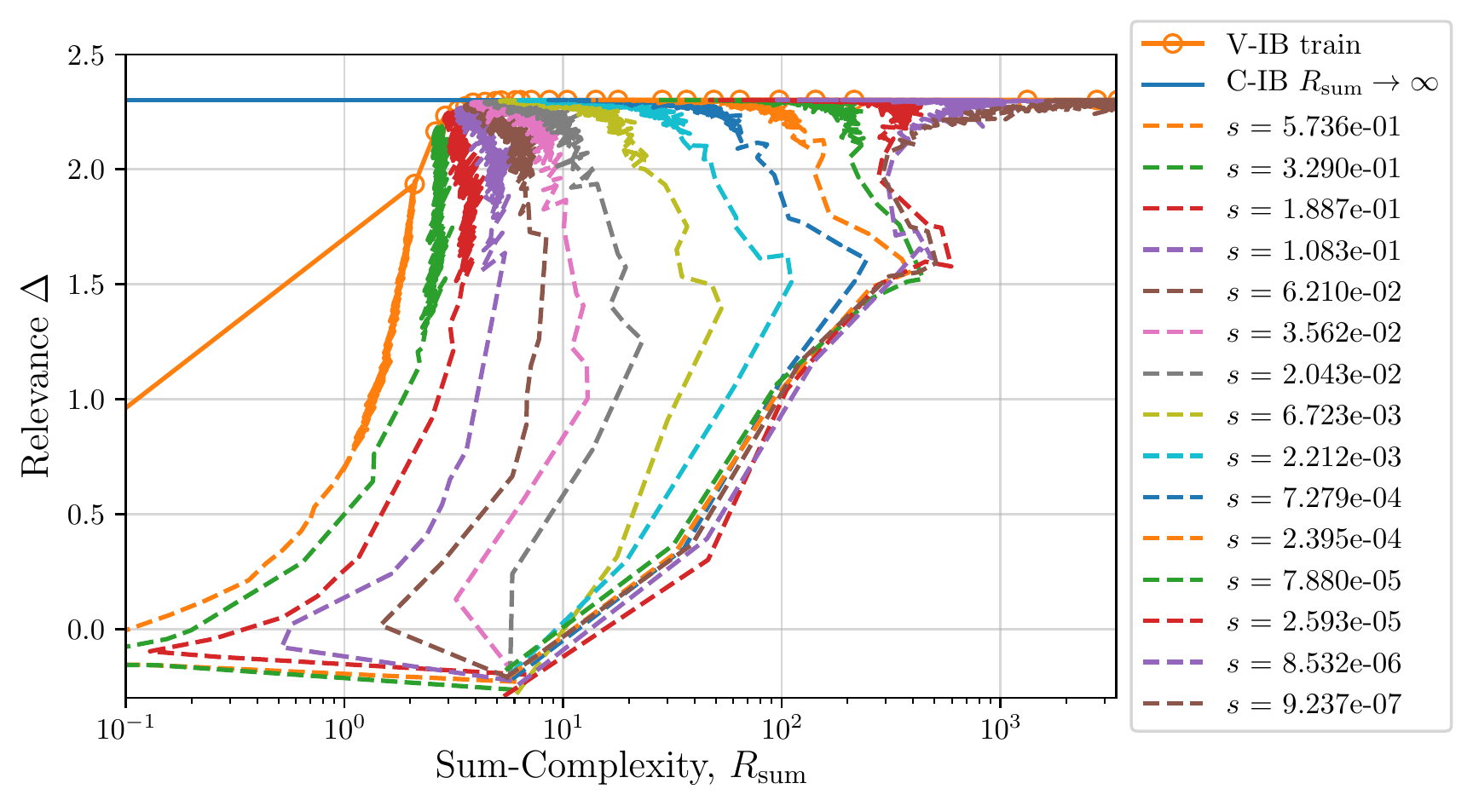}
\captionof{figure}{Evolution of the relevance-complexity pairs during training with D-VIB for different values of the regularization parameter $s$ over the relevance-complexity plane.}
\label{fig:MNIST_IB_TrainTest_traj}
\vspace{-6mm}        
\end{figure}

\subsection{Evolution of the Relevance-Complexity pairs}
In this section, we consider the evolution of the achievable relevance-complexity pairs over the relevance-complexity plane during the training phase when using the D-VIB as given in Algorithm~\ref{algo:VDIB}. We are interested in the trajectory followed by the pairs for different regularization parameter value $s$ as the DNN parameters are updated from a random initialization until their convergence.  In particular, for a given value $s$, let $(\boldsymbol\theta^{t}, \boldsymbol\phi^{t}, \boldsymbol \varphi^t )$ denote the  DNNs parameters determining encoders $P_{\theta_k^t}(U_k|X_k)$, decoders $Q_{\phi_{\mc K}^t}(Y|U_{\mc K})$ and $Q_{\phi_k^{t}}(Y|U_k)$, and priors $Q_{\varphi_k^{t}}$ at iteration $t$  of Algorithm~\ref{algo:VDIB}. The achievable relevance pair $(\Delta_s^{t},R_{s}^{t})$ on the $(\Delta, R_{\mathrm{sum}})$ plane at iteration $t$ with such estimator can be calculated by evaluating \eqref{eq:Dparam}-\eqref{eq:R1param} for the current DNN parameters. 

For simplicity, we consider the standard MNIST database with $70.000$ labeled images of handwritten digits between $\{0,\ldots, 9\}$, observed under a single view consisting of a single view of $28\times 28$ grayscale pixels per sample, i.e., $K=1$. We consider a parameterization as in Section~\ref{secc:MultiviewMnist} with an encoder with two layers of 1024 neurons each with ReLu activations, whose output is mapped to generate the mean and covariance of a multivariate Gaussian embedding of size 256. The decoder consists of a single layer of 256  activations which generate a conditional probability of size $10$. The DNN parameters are initialized with Glorot's initialization \cite{Glorot10understandingthe} and updated with Algorithm~\ref{algo:VDIB} until convergence for different values of the regularization parameter $s$. 

Figure~\ref{fig:MNIST_IB_TrainTest_traj} shows the evolution of the pairs $(\Delta_s^{t},R_{s}^{t})$ for different values of the regularization parameter $s$ during training starting from the achievable pairs with randomly initialized weights around $(\Delta^{t_0}_s, R^{t_0}_{s})\simeq (-0.2,0.5)$ until their converge to the corresponding achievable point on the V-IB training curve. Three different type of trajectories can be observed depending on the value of $s$: i)
 For  $s\in [5.736 e -01, 2.043e-02)$ (curves on left), the complexity decreases without significant improvement in the relevance. Then relevance improves at the cost of higher complexity and starts reducing complexity just before convergence on the point on the V-IB train curve; ii) for $s\in [2.043e-02, 2.594e-05)$ the algorithm directly starts improving relevance and starts reducing complexity before convergence; iii) For $s\leq  2.594e-05$ (curves on right) relevance improves by increasing complexity, then complexity is reduced while improving relevance, and finally they continue improving relevance by increasing complexity until convergence.

The evolution of the relevance-complexity pairs at different layers of a DNN has been studied in \cite{Shwartz-ZivT17} and \cite{SBDAKTC18}, where relevance is measured as the mutual information between $Y$ and the output at layer $l$, and complexity is given by the mutual information between the observed data and the output. Under tanh activations, only curves of type ii) are observed in~\cite{Shwartz-ZivT17} while a different type of trajectory has been observed under ReLu activations in \cite{SBDAKTC18}. Figure~\ref{fig:MNIST_IB_TrainTest_traj} shows that very different types of trajectories can result for the same DNN architecture, i.e., weights and activations, depending on the regularization parameter. Figure~\ref{fig:MNIST_IB_TrainTest_traj} shows a phase transition between the three types of trajectories on the plane. The study of the trajectories followed by the algorithm for different activations  is left as future work.

\section{Proofs}

\subsection{Proof of Theorem~\ref{th:Comp_Rel_region_DM}}\label{app:Comp_Rel_region_DM}

In this proof, we show the connection between the distributed learning problem under study and the $K$-encoder CEO problem studied in~\cite{Courtade2014LogLoss}. For the $K$-encoder CEO problem, let us consider $K$ encoding functions  $\phi_k:\mc X_k\rightarrow \mc M_k^{(n)}$ satisfying the complexity constraint~\eqref{eq:Complexity_definition} and a decoding function $\tilde{\psi}: \mathcal{M}^{(n)}_1\times \hdots \times \mathcal{M}^{(n)}_K \rightarrow \mathcal{\hat{Y}}^n$, that produces a probabilistic estimate of $Y$ from the outputs of the encoders, i.e., $\mathcal{\hat{Y}}^n$ is the set of distributions on $\mc Y$. The quality of the estimation is measured in terms of the average log-loss. 

\begin{definition}
A tuple $(D,R_1,\ldots, R_K)$ is said to be achievable in the $K$-encoder CEO problem for $P_{X_{\mc K},Y}$ for which the Markov chain~\eqref{eq:MKChain_pmf} holds, if 
there exists a length $n$, encoders $\phi_k$ for $k\in \mc K$ and a decoder $\tilde{\psi}$, 
such that
\begin{align}
D & \geq  \mathrm{E}\left[\frac{1}{n}\log\frac{1}{\hat{P}_{Y^n|J_{\mc K}}(Y^n|\phi_1(X_1^n),\hdots, \phi_K(X_K^n))}\right], \\
R_k &\geq \frac{1}{n}\log |\phi_k(X_k^n)| \quad \text{for all}\:\:\: k \in \mc K.
%where $\hat{P}_{Y^n|J}(y^n|\phi(x^n))$, $y^n\in \mc Y^n$, $x^n\in \mc X^n$ is the value of this distribution evaluated for %the outcome $j=\phi(x^n)$. 
\end{align}
The rate-distortion region $\mathcal{RD}_{\mathrm{CEO}}$ is given by the closure of all achievable tuples $(D,R_1,\ldots, R_K)$. \qed
\end{definition}

The following lemma shows that the minimum average logarithmic loss is the conditional entropy of $Y$ given the descriptions. The result is essentially equivalent to \cite[Lemma 1]{Courtade2014LogLoss} and it is provided for completeness.
\begin{lemma}\label{lem:LL_Relevance}
Let us consider $P_{X_{\mc K}, Y}$ and the encoders $J_k = \phi_k(X_k^n)$, $k\in \mc K$ and the decoder $\hat{Y}^n = \tilde{\psi}(J_{\mc K})$. Then,
\begin{align}
\mathrm{E}[\ell_{\mathrm{log}}(Y^n,\hat{Y}^n)]\geq H(Y^n|J_{\mc K}),
\end{align} 
with equality if and only if $\tilde{\psi}(J_{\mc K}) = \{P_{Y^n|J_{\mc K}}(y^n|J_{\mc K})\}_{y^n\in \mc Y^n}$. 
\end{lemma}
\begin{proof}
Let $Z := (J_1,\ldots, J_K)$ be the argument of $\tilde{\psi}$ and $\hat{P}(y^n|z)$ be a distribution on $\mc{Y}^n$. We have for $Z=z$:
\begin{align}
\mathrm{E}[\ell_{\mathrm{log}}(Y^n,\hat{Y}^n)|Z=z]
&= \sum_{y^n\in\mc{Y}^n}P(y^n|z)\log\left(\frac{1}{\hat{P}(y^n|z)}\right)\\
&= \sum_{y^n\in\mc{Y}^n}P(y^n|z)\log\left(\frac{P(y^n|z)}{\hat{P}(y^n|z)}\right) + H(Y^n|Z=z)\\
&= D_{\mathrm{KL}}(P(y^n|z)\|\hat{P}(y^n|z))+ H(Y^n|Z=z)\\
&\geq H(Y^n|Z=z),\label{eq:LowerBound}
\end{align}
where \eqref{eq:LowerBound} is due to the non-negativity of the KL divergence and the equality holds if and only if for $\hat{P}(y^n|z) = P(y^n|z)$ where $ P(y^n|z) = \mathrm{Pr}\{Y^n=y^n|Z=z\}$ for all $z$ and $y^n\in \mc Y^n$. Averaging over  $Z$ completes the proof.
\end{proof}

Essentially, Lemma~\ref{lem:LL_Relevance} states that  minimizing the average log-loss is equivalent to maximizing relevance as given in~\eqref{eq:DeltaConstr}. Formally, the connection between the distributed learning problem under study and the $K$-encoder CEO problem studied in~\cite{Courtade2014LogLoss} can be formulated as stated next.
\begin{proposition}\label{prop:Eqreg}
A tuple $(\Delta,R_1,\ldots, R_K)\in \mc {RI}_{\mathrm{DIB}}$ if and only if $(H(Y)-\Delta,R_1,\ldots, R_K)\in \mc {RD}_{\mathrm{CEO}}$.
\end{proposition}
\begin{proof}
Let the tuple $(\Delta,R_1,\ldots, R_K)\in \mc {RI}_{\mathrm{DIB}}$ be achievable for some encoders $\phi_k$, i.e., \eqref{eq:DeltaConstr} and \eqref{eq:RConstr} hold. It follows by Lemma~\ref{lem:LL_Relevance} that by letting the decoding function $\tilde{\psi}(J_{\mc K}) = \{P_{Y^n|J_{\mc K}}(y^n|J_{\mc K})\}$, we have $\mathrm{E}[\ell_{\mathrm{log}}(Y^n,\hat{Y}^n)|J_{\mc K}]
= H(Y^n|J_{\mc K})$, and hence $(H(Y)\!-\!\Delta,R_1,\ldots, R_K)\in \mc {RD}_{\mathrm{CEO}}$. 

Conversely, assume the tuple $(D,R_1,\ldots, R_K)\in \mc {RD}_{\mathrm{CEO}}$ is achievable. It follows by Lemma~\ref{lem:LL_Relevance} that  $H(Y)-D \leq H(Y^n)- H(Y^n|J_{\mc K})=  I(Y^n;J_{\mc K})$, which implies $(\Delta,R_1,\ldots, R_K)\in \mc {RI}_{\mathrm{DIB}}$ with $\Delta = H(Y)-D$.
\end{proof}

The characterization of rate-distortion region $\mc {R}_{\mathrm{CEO}}$ has been established recently in \cite[Theorem 10]{Courtade2014LogLoss}. The proof of the theorem is completed by noting that Proposition \ref{prop:Eqreg} implies that the result in~\cite[Theorem 10]{Courtade2014LogLoss} can be applied to characterize the region $\mathcal{RI}_{\mathrm{DIB}}$ as given in Theorem~\ref{th:Comp_Rel_region_DM}. \qed

\subsection{Proof of Proposition~\ref{prop:Sum-RateRegion}}\label{app:Sum-RateRegion}
For simplicity of exposition, the proof is given for the case $K=2$ encoders. The proof for $K>2$ follows similarly. By the definition of $\mc{RI}_{\mathrm{DIB}}^{\mathrm{sum}}$, the  tuple $(\Delta, R_{\mathrm{sum}})\in \mathds{R}_+^2$ is achievable for some random variables $Y,X_1,X_2,U_1, U_2$ with joint pmf satisfying \eqref{eq:distributionFactortization}, if it holds that
\begin{align}
\Delta&\leq I(Y;U_1,U_2)\label{eq:inequaltiy1}\\
\Delta&\leq R_1 - I(X_1;U_1|Y) + I(Y;U_2)\\
\Delta&\leq R_2 - I(X_2;U_2|Y) + I(Y;U_1)\\
\Delta&\leq R_1+R_2 -I(X_1;U_1|Y)-I(X_2;U_2|Y)\\
R_1+R_2&\leq R_{\mathrm{sum}}.\label{eq:inequaltiy_last}
\end{align}
The application of the Fourier-Motzkin elimination to project out $R_1$ and $R_2$ reduces the system on inequalities \eqref{eq:inequaltiy1}-\eqref{eq:inequaltiy_last} to  the following system of inequalities
\begin{align}
\Delta&\leq I(Y;U_1,U_2)\label{eq:inequaltiy_last_reduced_1}\\
\Delta&\leq R_{\mathrm{sum}} -I(X_1;U_1|Y)-I(X_2;U_2|Y)\label{eq:inequaltiy_last_reduced_2}\\
2\Delta&\leq R_{\mathrm{sum}} - I(X_1;U_1|Y) - I(X_2;U_2|Y)\\
&\quad + I(Y;U_1) + I(Y;U_2)\label{eq:inequaltiy_last_reduced_3}
\end{align}
It follows due to the Markov chain $U_1\mkv X_1\mkv Y \mkv X_2\mkv U_2$ that we have $I(Y;U_1,U_2)\leq I(Y;U_1)+ I(Y;U_2)$. Therefore, inequality \eqref{eq:inequaltiy_last_reduced_3} is redundant as it is implied by \eqref{eq:inequaltiy_last_reduced_1} and \eqref{eq:inequaltiy_last_reduced_2}. 
This completes the proof. \qed

\subsection{Proof of Proposition~\ref{prop:param}}\label{app:param}
%\iffalse 
Suppose that $\dv P^*$ yields the maximum in \eqref{eq:Dparam}. Then, 
\begin{align}
(1+s)\Delta_{s}
& = (1+sK)H(Y)+sR_{s}+\mc L_{s}(\dv P^*)\\
&=(1+sK) H(Y) + sR_{ s}\label{eq:Ldefinition}
\\
&\;
+\left(-H(Y|U_{\mc K}^*) - s \sum_{k=1}^K [H(Y|U^*_k) + I(X_k;U_k^*)] \right)
%\nonumber
\\
 &=(1+sK) H(Y) + sR_{ s}\label{eq:MKchain}
\\
&\;
+(-H(Y|U_{\mc K}^*) - s(R_s  - I(Y;U^*_{\mc K})+KH(Y)) )
%\nonumber
\\
&=(1+s) I(Y;U_{\mc K}^*)\\
 &\leq (1+s)\Delta(R_{s}, P_{X_{\mc K},Y}), \label{eq:part1}
\end{align}
where \eqref{eq:Ldefinition} is due to the definition of $\mc L_s(\dv P)$ in \eqref{eq:CostF}; \eqref{eq:MKchain} follows since we have $
\sum_{k=1}^K [I(X_k;U_k^*) + H(Y|U_k^*)] = R_s  - I(Y;U^*_{\mc K})+KH(Y)$ from the definition of $R_s$ in~\eqref{eq:R1param}; and \eqref{eq:part1} follows from the definition in \eqref{eq:RelevanceSumComplexityFunction}.

Conversely, if $\dv P^*$ is the solution to the maximization in the function $\Delta(R_{\mathrm{sum}}, P_{X_{\mc K},Y})$ in~\eqref{eq:RelevanceSumComplexityFunction} such that $\Delta(R_{\mathrm{sum}}, P_{X_{\mc K},Y}) = \Delta_s$, then $\Delta_s\leq I(Y;U_{\mc{K}}^*)$ and $\Delta_s\leq R -\sum_{k=1}^K I(X_k;U^*_k|Y)$ and we have, for any $s\geq 0$, that
\begingroup
\allowdisplaybreaks
\begin{align}
\Delta(R_{\mathrm{sum}},P_{X_{\mc K}, Y}) 
&=\Delta_s\nonumber\\
&\leq   \Delta_s -(\Delta_s- I(Y;U_{\mc{K}}^*))
-s\left(\Delta_s- R_{\mathrm{sum}} +\sum_{k=1}^K I(X_k;U^*_k|Y)\right)\nonumber\\
&=  I(Y;U_{\mc{K}}^*)-s\Delta_s +s R_{\mathrm{sum}}  -s\sum_{k=1}^K I(X_k;U^*_k|Y)\nonumber\\
&= H(Y) - s\Delta_s + sR_{\mathrm{sum}}-H(Y|U^*_{\mc K}) \label{eq:MK}
 - s\sum_{k=1}^K[ I(X_k;U^*_k) +H(Y|U_k^*)]+ sKH(Y) 
% \nonumber
 \\
&\leq  H(Y) - s\Delta_s + sR_{\mathrm{sum}}  + \mc L_{s}^* + sKH(Y) \label{eq:OptL}\\
&= H(Y) - s\Delta_s  + sR_{\mathrm{sum}} + sKH(Y)
- ((1+sK)H(Y) + sR_s -(1+s)\Delta_s)\label{eq:LagrangianEq}\\
&= \Delta_{s} + s(R_{\mathrm{sum}} -R_{s}) \label{eq:LagrangianEq_2},
\end{align}
\endgroup
where in \eqref{eq:MK} we have $\sum_{k=1}^KI(X_k;U_k|Y)  = -KH(Y)+\sum_{k=1}^KI(X_k;U_k)+H(Y|U_k)$ 
 due to the Markov chain $U_k - X_k - Y - (X_{\mc K\setminus k }, U_{\mc K\setminus k })$; \eqref{eq:OptL} follows since $\mc L^*_{s}$ is the maximum over all possible distributions $\dv P$ (not necessarily $\dv P^*$ maximizing $\Delta(R_{\mathrm{sum}}, P_{X_{\mc K},Y})$); and \eqref{eq:LagrangianEq} is due to~\eqref{eq:Dparam}. 

Finally, \eqref{eq:LagrangianEq_2} is valid for any $R_{\mathrm{sum}}\geq 0$ and $s\geq 0$. 
Given $s$, and hence $(\Delta_s, R_{s})$, letting $R = R_{s}$ yields
$\Delta(R_s, P_{X_{\mc K},Y}) \leq \Delta_{ s}$. 
Together with \eqref{eq:part1}, this completes the proof.
\qed
%\fi

\subsection{Proof of Lemma~\ref{lemma:QUpdate}}\label{app:QUpdate}
The proof follows by deriving the following bounds. For any  pmf $Q_{Y|Z}(y|z)$,  $y\in \mc Y$ and $z \in \mc Z$, e.g., $\mc Z = \mc U_{\mc K}$ or $\mc Z = \mc U_k$, proceeding similarly to~\eqref{eq:LowerBound} and averaging over $Z$, we have
\begin{align}
H(Y|Z) &= \mathds{E}[-\log Q_{Y|Z}(Y|Z)]-D_{\mathrm{KL}}(P_{Y|Z}\|Q_{Y|Z}).\nonumber
\end{align}
Similarly,  we have
\begin{align}
I(X_k;U_k) &= H(U_k)-H(U_k|X_k)\\
%&=\mathds{E}[-\log Q_{U_k}(U_k)]-D_{\mathrm{KL}}(P_{U_k}\|Q_{U_k}) - H(X_k|U_K)\\
&=D_{\mathrm{KL}}(P_{U_k|X_k}\|Q_{U_k}) -D_{\mathrm{KL}}(P_{U_k}\|Q_{U_k})
\end{align}

Applying these identities to~\eqref{eq:CostF}, we have
\begin{align}
\mc L_s(\dv P) &= \mc L^{\mathrm{VB}}_s(\dv P, \dv Q) + D_{\mathrm{KL}}(P_{Y|U_{\mc K}}||Q_{Y|U_{\mc K}})+s\sum_{k=1}^K (D_{\mathrm{KL}}(P_{Y|U_k}||Q_{Y|U_k})+ D_{\mathrm{KL}}(P_{U_{k}}||Q_{U_k})  )\nonumber\\
&\geq \mc L^{\mathrm{VB}}_s(\dv P, \dv Q),\label{eq:UpperBound}
\end{align}
where \eqref{eq:UpperBound} follows since the KL-divergence is always positive and equality is met iff $\dv Q^*$ is given as  \eqref{eq:Qstark} and \eqref{eq:Qstarall}.
\qed

\subsection{Proof of Theorem~\ref{th:GaussSumCap}}\label{app:GaussSumCap}

The proof of Theorem~\ref{th:GaussSumCap} relies on deriving an outer bound on the relevance-complexity region described by~\eqref{eq:ComplexityRelevanceFunction}, and showing that it is achievable with Gaussian pmfs and without time-sharing. In doing so, we use the technique of~\cite[Theorem 8]{Ekrem:TIT:14} which relies on the de Bruijn identity and the properties of Fisher information and MMSE. 

\vspace{-0.5em}
\begin{lemma}{\cite{Dembo:IT:91,Ekrem:TIT:14}}\label{lem:FI_Ineq}
Let $(\mathbf{X,Y})$  be a pair of random vectors with pmf $p(\mathbf{x},\mathbf{y})$. We have
\begin{align}
\log|(\pi e) \mathbf{J}^{-1}(\mathbf{X}|\mathbf{Y})|\leq h(\mathbf{X}|\mathbf{Y})\leq\log|(\pi e) \mathrm{mmse}(\mathbf{X}|\mathbf{Y})|,\nonumber
\end{align}
where the conditional Fischer information matrix is defined as$
\mathbf{J}(\mathbf{X}|\mathbf{Y}) := \mathrm{E}[\nabla \log p(\mathbf{X}|\mathbf{Y})\nabla\log p(\mathbf{X}|\mathbf{Y})^\dagger]$,
and the minimum mean square error (MMSE) matrix is 
$\mathrm{mmse}(\mathbf{X}|\mathbf{Y}) := \mathrm{E}[(\dv X-\mathrm{E}[\dv X|\dv Y])(\dv X-\mathrm{E}[\dv X|\dv Y])^\dagger]$. 
\end{lemma}

%\begin{lemma}\label{lemm:Brujin}
%Let $(\mathbf{V}_1,\mathbf{V}_2)$ be a random vector with finite second moments and $\mathbf{N}\!\sim\!\mc{CN}(\dv 0, \boldsymbol\Lambda_N)$ independent of $(\mathbf{V}_1,\mathbf{V}_2)$. Then
%\begin{align}
%\mathrm{mmse}(\mathbf{V}_2|\mathbf{V}_1,\mathbf{V}_2+\mathbf{N}) = \boldsymbol\Lambda_N -\boldsymbol\Lambda_N\mathbf{J}(\mathbf{V}_2+\mathbf{N}|\mathbf{V}_1)\boldsymbol\Lambda_N\nonumber.
%\end{align}
%\end{lemma}

First, we outer bound the relevance-complexity region in Theorem~\ref{th:Comp_Rel_region_DM} for $(\dv Y,\dv X_{\mc K})$ as in \eqref{mimo-gaussian-model}. 
For $ t\in \mc{T}$ and fixed $\prod_{k=1}^{K}p(\mathbf{u}_k|\mathbf{x}_k,t)$, choose $\mathbf{\Omega}_{k,t}$, $k = 1,\ldots, K$ satisfying $\mathbf{0}\preceq\mathbf{\Omega}_{k,t}\preceq\mathbf{\Sigma}_{k}^{-1}$ such that 
\begin{align}
\mathrm{mmse}(\mathbf{Y}_k|\mathbf{X}, \mathbf{U}_{k,t},t) = \mathbf{\Sigma}_{k}-\mathbf{\Sigma}_{k}\mathbf{\Omega}_{k,t}\mathbf{\Sigma}_{k}.\label{eq:covB}
\end{align}
Such $\mathbf{\Omega}_{k,t}$ always exists since $ \mathbf{0}\preceq\mathrm{mmse}(\mathbf{X}_k|\mathbf{Y},\mathbf{U}_{k,t},t)\preceq \mathbf{\Sigma}_{k}^{-1}$, for all $t\in \mc T$, and $k\in \mc K$. We have from \eqref{eq:ComplexityRelevanceFunction},
\begin{align}
I(\mathbf{X}_k;\mathbf{U}_k|\mathbf{Y},t)
%& = \log|(\pi e)\boldsymbol\Sigma_{k}| -h(\mathbf{X}_k|\mathbf{Y},\mathbf{U}_{k,t},T=t)\\
& \geq \log|\boldsymbol\Sigma_{k}| -\log|\mathrm{mmse}(\mathbf{X}_k|\mathbf{Y},\mathbf{U}_{k,t},t) |
\nonumber\\
&
= - \log|\dv I-\mathbf{\Sigma}_{k}^{1/2}\mathbf{\Omega}_{k,t}\mathbf{\Sigma}_{k}^{1/2}|,\label{eq:firstIneq}
\end{align}
where the inequality is due to Lemma~\ref{lem:FI_Ineq}, and \eqref{eq:firstIneq} is due to \eqref{eq:covB}.

On the other hand, we have
\begin{align}
I(\mathbf{Y};\mathbf{U}_{S^c,t}|t)
%=h(\mathbf{Y})-h(\mathbf{Y}|\mathbf{U}_{S^c,t},t)\\
&\leq \log|\mathbf{\Sigma}_{\mb y} |-\log|\mathbf{J}^{-1}(\mathbf{Y}|\mathbf{U}_{S^c,t},t)|\label{eq:FI_Ineq}\\
& = \log 
\left| \sum_{k\in\mathcal{S}^{c}}\mathbf{\Sigma}_{\mb y}^{1/2}\mathbf{H}_{k}^{\dagger}
\mathbf{\Omega}_{k,t}
\mathbf{H}_{k}\mathbf{\Sigma}_{\mb y}^{1/2}+\mathbf{I}\right|\label{eq:secondtIneq},
\end{align}
where \eqref{eq:FI_Ineq} is due to Lemma~\ref{lem:FI_Ineq}; and \eqref{eq:secondtIneq} is due to to the following equality connecting the MMSE matrix \eqref{eq:covB} and the Fisher information as in
\cite{Ekrem:IT:2014:OuterBoundCEO, DBLP:journals/corr/ZhouX0C16,  EZCS:IT:2017}, proven below:
\begin{align}
\mathbf{J}(\mathbf{Y}|\mathbf{U}_{S^c,t},t) = \sum_{k\in\mathcal{S}^{c}}\mathbf{H}_{k}^{\dagger}
\mathbf{\Omega}_{k,t}
\mathbf{H}_{k}+\mathbf{\Sigma}_{\mb y}^{-1}\label{eq:Fischerequality}.
\end{align}

In order to show \eqref{eq:Fischerequality}, we use de Brujin identity to relate the Fisher information with the MMSE as given in the following lemma from \cite{Ekrem:TIT:14}. 
\begin{lemma}\label{lemm:Brujin}
Let $(\mathbf{V}_1,\mathbf{V}_2)$ be a random vector with finite second moments and $\mathbf{N}\!\sim\!\mc{CN}(\dv 0, \boldsymbol\Sigma_N)$ independent of $(\mathbf{V}_1,\mathbf{V}_2)$. Then
\begin{align}
\mathrm{mmse}(\mathbf{V}_2|\mathbf{V}_1,\mathbf{V}_2+\mathbf{N}) = \boldsymbol\Sigma_N -\boldsymbol\Sigma_N\mathbf{J}(\mathbf{V}_2+\mathbf{N}|\mathbf{V}_1)\boldsymbol\Sigma_N\nonumber.
\end{align}
\end{lemma}

From the MMSE  of Gaussian random vectors~\cite{elGamal:book},
\begin{align}
\mathbf{Y} = \mathrm{E}[\mathbf{Y}|\mathbf{X}_{\mathcal{S}^c}]+\mathbf{Z}_{\mathcal{S}^c} = \sum_{k\in \mathcal{S}^c}\mathbf{G}_{k}\mathbf{X}_{k} +\mathbf{Z}_{\mathcal{S}^c},
\end{align}
where $\mathbf{G}_{k} = \dv \Sigma_{\dv y | \dv x_{\mc{S}^c}}  \mathbf{H}^{\dagger}_{k}\mathbf{\Sigma}_{k}^{-1}$ and $\mathbf{Z}_{\mathcal{S}^c}\sim\mathcal{CN}(\mathbf{0},\dv \Sigma_{\dv y | \dv x_{\mc{S}^c}} )$,  and
\begin{align}
\dv \Sigma_{\dv  y | \dv x_{\mc{S}^c}}^{-1} =   \mathbf{\Sigma}_{\mb y}^{-1} +\sum_{k\in \mathcal{S}^c}\mathbf{H}_{k}^{\dagger}\mathbf{\Sigma}_{k}^{-1}\mathbf{H}_{k}.\label{eq:CovZ_xy}
\end{align}
Note that $\mathbf{Z}_{\mathcal{S}^c}$ is independent of $\mathbf{Y}_{\mathcal{S}^c}$ due to the orthogonality principle of the MMSE and its Gaussian distribution. Hence, it is also independent of $\mathbf{U}_{\mathcal{S}^c,q}$.
We have
\begin{align}
\text{mmse}
\left(
 \sum_{k\in \mathcal{S}^c}\mathbf{G}_{k}
\mathbf{X}_k
  \Big|\mathbf{Y}, \mathbf{U}_{\mathcal{S}^c,t},t 
  \right)  
&= \sum_{k\in \mathcal{S}^c}\mathbf{G}_{k}
\text{mmse}\left(\mathbf{X}_k
  |\mathbf{Y}, \mathbf{U}_{\mathcal{S}^c,t},t \right)\mathbf{G}_{k}^{\dagger}\label{eq:CrossTerms}\\
&= \dv \Sigma_{\dv y | \dv x_{\mc{S}^c}} \sum_{k\in \mathcal{S}^c}\mathbf{H}_{k}^{\dagger}
\left(\mathbf{\Sigma}_{k}^{-1}-\mathbf{\Omega}_{k} \right)\mathbf{H}_{k}\dv \Sigma_{\dv y | \dv x_{\mc{S}^c}} \label{eq:CovSubs},
\end{align}
where \eqref{eq:CrossTerms} follows since the cross terms are zero due to the Markov chain 
$(\mathbf{U}_{k,t},\mathbf{X}_k)\mkv \mathbf{Y} \mkv (\mathbf{U}_{\mathcal{K}/k,t},\mathbf{X}_{\mathcal{K}/k})$, 
see \cite[Appendix V]{Ekrem:IT:2014:OuterBoundCEO};
and \eqref{eq:CovSubs} follows due to \eqref{eq:covB} and $\mathbf{G}_{k}$.
Finally, 
\begin{align}
\mathbf{J}(\mathbf{Y}|\mathbf{U}_{S^c,t},t)&=\dv \Sigma_{\dv y | \dv x_{\mc{S}^c}}^{-1} 
- \dv \Sigma_{\dv y | \dv x_{\mc{S}^c}}^{-1} \text{mmse}
\left(
 \sum_{k\in \mathcal{S}^c}\mathbf{G}_{k}\mathbf{X}_{k}
  \Big|\mathbf{Y}, \mathbf{U}_{\mathcal{S}^c,t},t 
  \right)  \dv \Sigma_{\dv y | \dv x_{\mc{S}^c}}^{-1} \label{eq:LemmaBrujin}\\
&=\dv \Sigma_{\dv y | \dv x_{\mc{S}^c}}^{-1} -  
 \sum_{k\in \mathcal{S}^c}\mathbf{H}_{k}^{\dagger}
\left(\mathbf{\Sigma}_{k}^{-1}-\mathbf{\Omega}_{k,t} \right)\mathbf{H}_{k}\label{eq:MMSEsubs}\\
 &=\mathbf{\Sigma}_{\mb y}^{-1} + 
 \sum_{k\in \mathcal{S}^c}\mathbf{H}_{k}^{\dagger}
\mathbf{\Omega}_{k,t}\mathbf{H}_{k}\label{eq:MMSEsubs_2},
 \end{align}
where \eqref{eq:LemmaBrujin} is due to Lemma~\ref{lemm:Brujin};
\eqref{eq:MMSEsubs} is due to \eqref{eq:CovSubs}; and \eqref{eq:MMSEsubs_2} follows due to \eqref{eq:CovZ_xy}.

Then, averaging over the time sharing $T$ and letting $\bar{\mathbf{\Omega}}_k:= \sum_{t\in \mathcal{T}}p(t)\mathbf{\Omega}_{k,t}$. Then, we have from \eqref{eq:firstIneq}
\begin{align}
I(\mathbf{X}_k;\mathbf{U}_k|\mathbf{Y},T) 
%&= \sum_{t\in \mathcal{T}}p(t)I(\mathbf{X}_k;\mathbf{U}_k|\dv Y,t)\\
&\geq -  \sum_{t\in \mathcal{T}}p(t) \log|\dv I-\mathbf{\Sigma}_{k}^{1/2}\mathbf{\Omega}_{k,t}\mathbf{\Sigma}_{k}^{1/2}|\nonumber \\%\label{eq:logDetProp1}\\
%&\geq - \log|\dv I-\mathbf{\Sigma}_{k}^{1/2}\sum_{t\in \mathcal{T}}p(t)\mathbf{\Omega}_{k,t}\mathbf{\Sigma}_{k}^{1/2}|\label{eq:logDetProp2}\nonumber\\
&\geq-\log|\dv I-\mathbf{\Sigma}_{k}^{1/2}\bar{\mathbf{\Omega}}_k\mathbf{\Sigma}_{k}^{1/2}|,\label{eq:logDetProp2}
\end{align}
where  \eqref{eq:logDetProp2} follows from the concavity of the log-det function and Jensen's inequality.

Similarly, from \eqref{eq:secondtIneq} and Jensen's Inequality we have
\begin{align}
I(\mathbf{Y};\mathbf{U}_{S^c}|T)
&\leq
\log\left| \sum_{k\in\mathcal{S}^{c}}\mathbf{\Sigma}_{\mb y}^{1/2}\mathbf{H}_{k}^{\dagger}
\bar{\mathbf{\Omega}}_{k}
\mathbf{H}_{k}\mathbf{\Sigma}_{\mb y}^{1/2}+\mathbf{I}\right|
\label{eq:secondtIneq_4}. 
\end{align}

The outer bound on $\mc{RI}_{\mathrm{DIB}}$ is obtained by applying~\eqref{eq:logDetProp2} and~\eqref{eq:secondtIneq_4} in~\eqref{eq:ComplexityRelevanceFunction}, noting that $\dv \Omega_k = \sum_{t\in \mathcal{T}}p(t) \dv \Omega_{k,t} \preceq\mathbf{\Sigma}_{k}^{-1}$ since $\mathbf{0} \preceq \dv \Omega_{k,t} \preceq\mathbf{\Sigma}_{k}^{-1}$, and taking the union over $\dv \Omega_k$ satisfying $\mathbf{0} \preceq \dv \Omega_k \preceq\mathbf{\Sigma}_{k}^{-1}$.

Finally, the proof is completed by noting that the outer bound is achieved with $T= \emptyset$ and multivariate Gaussian  distributions $p^{*}(\dv u_k|\dv x_k,t) = \mc{CN}(\dv x_k,  \mathbf{\Sigma}_{k}^{1/2}(\dv \Omega_k-\dv I)\mathbf{\Sigma}_{k}^{1/2}  ) $.
\qed

\subsection{Derivation of Algorithm~\ref{algo:BA_Gauss}}\label{app:BADIVGauss}
In this section, we derive the update rules in Algorithm~\ref{algo:BA_Gauss} and show that the Gaussian distribution is invariant to the update rules in Algorithm~\ref{algo:BA_DMC}, in line with Theorem~\ref{th:GaussSumCap}.

First, we recall that if $(\dv X_1,\dv X_2)$ are jointly Gaussian, then
\begin{align}
P_{\dv X_2|\dv X_1 = \dv x_1} = \mc{CN}(\boldsymbol\mu_{\dv x_2|\dv x_1},\dv\Sigma_{\dv x_2|\dv x_1}),
\end{align}
where
$\boldsymbol\mu_{\dv x_2|\dv x_1}:= \dv K_{\dv x_2|\dv x_1}\dv x_1$, with $\dv K_{\dv x_2|\dv x_1}:=\dv\Sigma_{\dv x_2,\dv x_1}\dv \Sigma_{\dv x_1}^{-1}$ .

Then, for $\dv Q^{(t+1)}$ computed as in \eqref{eq:Qstark} and \eqref{eq:Qstarall} from $\dv P^{(t)}$, which is a set of Gaussian distributions,  we have
\begin{align}
 Q^{(t+1)}_{\dv Y|\dv{u}_k} &= \mc {CN}(\boldsymbol \mu_{\dv y|\dv u_{k}^t} , \dv \Sigma_{\dv y|\dv u_{k}^t}),\nonumber\\
Q^{(t+1)}_{\dv Y|\dv{u}_{\mc K}} &= \mc {CN}(\boldsymbol \mu_{\dv y|\dv u_{\mc K}^t}, \dv \Sigma_{\dv y|\dv u_{\mc K}^t})\nonumber.
\end{align}

Next, we look at the update  $\dv P^{(t+1)}$ as in \eqref{eq:P_update} from given $\dv Q^{(t+1)}$. First, we have that $p(\dv u_{k}^t)$  is the marginal of $\dv U_{k}^t$, given by $\dv U_{k}^t\sim \mc{CN}(\dv 0,\dv \Sigma_{\dv u_{k}^t} )$  where $\dv \Sigma_{\dv u_{k}^t} = \dv A_{k}^t\dv \Sigma_{\dv x_k} \dv A_{k}^{t,H} + \dv\Sigma_{\dv z_{k}^t}$. 

Then, to compute $\psi_s(\dv u_k^t,\dv x_k)$, first, we note that 
\begin{align}
&E_{ U_{\mc K\setminus k}|x_k }[D_{\mathrm{KL}}(P_{ Y| U_{\mc K\setminus k},x_k}||Q_{Y| U_{\mc K\setminus k},u_k})]\label{eq:DistEq}
=D_{\mathrm{KL}}(P_{ Y, U_{\mc K\setminus k}|x_k}||Q_{ Y,U_{\mc K\setminus k}|u_k})\!-\!D_{\mathrm{KL}}(P_{ U_{\mc K\setminus k}|x_k}||Q_{ U_{\mc K\setminus k}|u_k}),
\end{align}
and that for two generic multivariate Gaussian distributions $P_1\sim\mc{CN}(\boldsymbol \mu_1,\dv \Sigma_1)$ and  $P_2\sim\mc{CN}(\boldsymbol \mu_2,\dv \Sigma_2)$ in $\mathds{C}^N$, the KL divergence is computed as %in~\eqref{eq:DistGaussi}.
\begin{align}
D_{\mathrm{KL}}(P_1\|P_2) =& \frac{1}{2}\left((\boldsymbol\mu_1-\boldsymbol\mu_2)^T\dv\Sigma_{2}^{-1}(\boldsymbol\mu_1-\boldsymbol\mu_2)
%\nonumber
\right.
%\\
%&
\left.
+\log |\dv\Sigma_2\dv\Sigma_1^{-1}| - d +\mathrm{tr}\{\dv \Sigma_2^{-1}\dv\Sigma_1\}\right).\label{eq:DistGaussi}
\end{align}
Applying \eqref{eq:DistEq} and \eqref{eq:DistGaussi} in \eqref{eq:P_update_psi} and noting that all involved distributions are Gaussian, it follows that $\psi_s(\dv u_k^t,\dv x_k)$ is a quadratic form. Then, since $p(\dv u_k^t)$ is Gaussian, the product $\log (p(\dv u_k^t)\exp(-\psi_s(\dv u_k^t,\dv x_k)))$ is also a quadratic form, and identifying constant, first and second order terms, we have
\begin{align}
\log p^{(t+1)}(\dv u_k|\dv x_k) 
&= Z(\dv x_k)+ (\dv u_k-\boldsymbol\mu_{\dv u_{k}^{t+1}| \dv x_k})^{H}\dv \Sigma_{\dv z_k^{t+1}}^{-1}
%\nonumber\\
%\cdot 
(\dv u_k-\boldsymbol\mu_{\dv u_{k}^{t+1}|\dv x_k}),
\end{align}
where $ Z(\dv x_k)$ is a normalization term independent of $\dv u_k$, and 
\begin{align}
\dv \Sigma_{\dv z_k^{t+1}}^{-1} &= \dv \Sigma_{\dv u_k^t}^{-1} + \dv K_{\dv y|\dv u_k^t}^H \dv \Sigma_{\dv y| \dv u_k^t}^{-1}\dv K_{\dv y|\dv u_k^t}
%\nonumber\\
%&
+\frac{1}{s}\dv K_{\dv y\dv u_{\mc K\setminus k}^t|\dv u_k^t}^H \dv \Sigma_{\dv y\dv u_{\mc K\setminus k}^t| \dv u_k^t}^{-1}\dv K_{\dv y\dv u_{\mc K\setminus k}^t|\dv u_k^t}
%\nonumber  \\
%&
- \frac{1}{s} \dv K_{\dv u_{\mc K\setminus k}^t|\dv u_k^t}^H \dv \Sigma_{\dv u_{\mc K\setminus k}^t| \dv u_k^t}^{-1}\dv K_{\dv u_{\mc K\setminus k}^t|\dv u_k^t}\label{eq:SecondOrder}, \\
\boldsymbol\mu_{\dv u_{k}^{t+1}|\dv x_k}&=\dv\Sigma_{\dv z_{k}^{t+1}}\left(  \dv K_{\dv y|\dv u_k^t}^H\dv \Sigma_{\dv y|\dv u_k^t}^{-1}\boldsymbol\mu_{\dv y|\dv x_k}\right.
%\nonumber\\
%&
+\frac{1}{s}\dv K_{\dv y,\dv u_{\mc K\setminus k}^t|\dv u_k^t} \dv \Sigma_{\dv y,\dv u_{\mc K\setminus k}^t|\dv u_k^t}^{-1}\boldsymbol\mu_{\dv y,\dv u_{\mc K\setminus k}^t|\dv x_k}
%\nonumber\\
%&
\left.-\frac{1}{s} \dv K_{\dv u_{\mc K\setminus k}^t|\dv u_k^t}\dv \Sigma_{\dv u_{\mc K\setminus k}^t|\dv u_k^t}^{-1}\boldsymbol\mu_{\dv u_{\mc K\setminus k}^t|\dv x_k}\label{eq:FirstOrder}\right). 
\end{align}
This shows that $p^{(t+1)}(\dv u_k|\dv x_k)$ is a multivariate Gaussian distribution and that $\dv U_{k}^{t+1}|\{\dv X_k=\dv x_k\}$ is also a multivariate Gaussian  distributed as $ \mc {CN}(\boldsymbol\mu_{\dv u_{k}^{t+1}|\dv x_k},\dv\Sigma_{\dv z_{k}^{t+1}})$.

Next, we simplify \eqref{eq:SecondOrder} and \eqref{eq:FirstOrder} to obtain the update rules \eqref{eq:SigmaUpdate} and \eqref{eq:AUpdate}. From the matrix inversion lemma, similarly to \cite{journals/jmlr/ChechikGTW05}, for $(\dv X_1,\dv X_2)$ jointly Gaussian  we have 
\begin{align}\label{eq:InvLemma1}
\dv \Sigma_{\dv x_2|\dv x_1}^{-1} = \dv\Sigma_{\dv x_2}^{-1} + \dv K_{\dv x_1|\dv x_2}^{H}\dv\Sigma_{\dv x_1|\dv x_2}^{-1}\dv K_{\dv x_1|\dv x_2}.
\end{align}
Applying \eqref{eq:InvLemma1}, in \eqref{eq:SecondOrder} we have
\begin{align}
\dv \Sigma_{\dv z_k^{t+1}}^{-1} 
%=& \dv \Sigma_{\dv u_k^t}^{-1} + (\dv \Sigma_{\dv u_k^t|\dv y}^{-1}  - \dv \Sigma_{\dv u_k^t}^{-1} )\nonumber\\
%&
%+\frac{1}{s}(\dv \Sigma_{\dv u_k^t|\dv y \dv u_{\mc K\setminus k}^t}^{-1}  - \dv \Sigma_{\dv u_k^t}^{-1} )
%- \frac{1}{s} (\dv \Sigma_{\dv u_k^t|\dv u_{\mc K\setminus k}^t}^{-1}  - \dv \Sigma_{\dv u_k^t}^{-1} )\nonumber\\
&=\dv \Sigma_{\dv u_k^t|\dv y}^{-1} +\frac{1}{s}\dv \Sigma_{\dv u_k^t|\dv y \dv u_{\mc K\setminus k}^t}^{-1} - \frac{1}{s} \dv \Sigma_{\dv u_k^t|\dv u_{\mc K\setminus k}^t}^{-1}   
\label{eq:SecondOrder_subLemma},\\ 
&=\left(1+\frac{1}{s}\right)\dv \Sigma_{\dv u_k^t|\dv y}^{-1}  - \frac{1}{s} \dv \Sigma_{\dv u_k^t|\dv u_{\mc K\setminus k}^t}^{-1},\label{eq:SecondOrder_subLemma_MK}
\end{align}
where  \eqref{eq:SecondOrder_subLemma_MK} is due to the Markov chain $\dv U_k\mkv \dv Y\mkv \dv U_{\mc K\setminus k }$.

Then, also from the matrix inversion lemma,  we have for jointly Gaussian $(\dv X_1,\dv X_2)$, 
\begin{align}\label{eq:InvLemma2}
\dv \Sigma_{\dv x_2|\dv x_1}^{-1} \dv \Sigma_{\dv x_2,\dv x_1} \dv\Sigma_{\dv x_1}^{-1} =\dv \Sigma_{\dv x_2}^{-1} \dv \Sigma_{\dv x_2,\dv x_1} \dv\Sigma_{\dv x_1|\dv x_2}^{-1}.
\end{align}
Applying \eqref{eq:InvLemma2} to \eqref{eq:FirstOrder}, for the first term in \eqref{eq:FirstOrder}, we have
\begin{align}
 \dv K_{\dv y|\dv u_k^t}^H\dv \Sigma_{\dv y|\dv u_k^t}^{-1}\boldsymbol\mu_{\dv y|\dv x_k}&=
%\dv \Sigma_{\dv u_k^t}^{-1}\dv\Sigma_{\dv y,\dv u_k}\dv \Sigma_{\dv y|\dv u_k^t}^{-1}\boldsymbol\mu_{\dv y|\dv x_k}\\
%&=
\dv \Sigma_{\dv u_k^t|\dv y}^{-1}\dv\Sigma_{\dv y,\dv u_k^t}\dv \Sigma_{\dv y}^{-1}\boldsymbol\mu_{\dv y|\dv x_k}\\
%&=
%\dv \Sigma_{\dv u_k^t|\dv x}^{-1}\dv  A_{k}^{t}\dv \Sigma_{\dv x_k,\dv y}\dv \Sigma_{\dv y}^{-1}\boldsymbol\mu_{\dv y|\dv x_k}\nonumber\\
&=
\dv \Sigma_{\dv u_k^t|\dv y}^{-1}\dv  A_{k}^{t}\dv \Sigma_{\dv x_k,\dv y}\dv \Sigma_{\dv y}^{-1}\dv\Sigma_{\dv y,\dv x_k}\dv \Sigma_{\dv x_k}^{-1}\dv x_k\nonumber\\
%\label{eq:eq:FirstOrder_eq_cov}\\
&=
\dv \Sigma_{\dv u_k^t|\dv y}^{-1}\dv  A_{k}^{t}(\dv I - \dv \Sigma_{\dv x_k|\dv y}\dv \Sigma_{\dv x_k}^{-1}) \dv x_k, \label{eq:eq:FirstOrder_eq_Invlemma1}
\end{align}
where $\dv\Sigma_{\dv y,\dv u_k^t}=\dv  A_{k}^{t}\dv \Sigma_{\dv x_k,\dv y}$; and \eqref{eq:eq:FirstOrder_eq_Invlemma1} is due to the definition of $\dv\Sigma_{\dv x_k |\dv y}$.
Similarly, for the second term in \eqref{eq:FirstOrder}, we have
\begin{align}
\dv K_{\dv y\dv u_{\mc K\setminus k}^t|\dv u_k^t} \dv \Sigma_{\dv y\dv u_{\mc K\setminus k}^t|\dv u_k^t}^{-1}\boldsymbol\mu_{\dv y,\dv u_{\mc K\setminus k}^t|\dv x_k}
%\nonumber\\
&=\dv \Sigma_{\dv u_k^t}^{-1}\dv\Sigma_{\dv y\dv u_{\mc K\setminus k}^t,\dv u_k}\dv \Sigma_{\dv y\dv u_{\mc K\setminus k}^t|\dv u_k^t}^{-1}\boldsymbol\mu_{\dv y\dv u_{\mc K\setminus k}^t|\dv x_k}\\
%&=
%\dv \Sigma_{\dv u_k^t|\dv y\dv u_{\mc K\setminus k}^t}^{-1}\dv\Sigma_{\dv y\dv u_{\mc K\setminus k}^t,\dv u_k}\dv \Sigma_{\dv y\dv u_{\mc K\setminus k}^t}^{-1}\boldsymbol\mu_{\dv y,\dv u_{\mc K\setminus k}^t|\dv x_k}\\
%&=
%\dv \Sigma_{\dv u_k^t|\dv y\dv u_{\mc K\setminus k}^t}^{-1}\dv  A_{k}^{t}\dv \Sigma_{\dv x_k,\dv y\dv u_{\mc K\setminus k}^t}\dv \Sigma_{\dv y\dv u_{\mc K\setminus k}^t}^{-1}\boldsymbol\mu_{\dv y,\dv u_{\mc K\setminus k}^t|\dv x_k}\nonumber\\
%&=
%\dv \Sigma_{\dv u_k^t|\dv y\dv u_{\mc K\setminus k}^t}^{-1}\dv  A_{k}^{t}\dv \Sigma_{\dv x_k,\dv y\dv u_{\mc K\setminus k}^t}\dv \Sigma_{\dv y\dv u_{\mc K\setminus k}^t}^{-1}\dv\Sigma_{\dv x_k,\dv y\dv u_{\mc K\setminus k}^t}\dv \Sigma_{\dv x_k}^{-1}\dv x_k\nonumber\\
&=
\dv \Sigma_{\dv u_k^t|\dv y\dv u_{\mc K\setminus k}^t}^{-1}\dv  A_{k}^{t}(\dv I - \dv \Sigma_{\dv x_k|\dv y\dv u_{\mc K\setminus k}^t}\dv \Sigma_{\dv x_k}^{-1}) \dv x_k,\\
&=
\dv \Sigma_{\dv u_k^t|\dv y}^{-1}\dv  A_{k}^{t}(\dv I - \dv \Sigma_{\dv x_k|\dv y}\dv \Sigma_{\dv x_k}^{-1}) \dv x_k, \label{eq:eq:FirstOrder_eq_MK_chain}
\end{align}
where we use $\dv\Sigma_{\dv u_k^t,\dv y\dv u_{\mc K\setminus k}^t}=\dv  A_{k}^t\dv \Sigma_{\dv x_k,\dv y\dv u_{\mc K\setminus k}^t}$; and \eqref{eq:eq:FirstOrder_eq_MK_chain} is due to the Markov chain $\dv U_k\mkv \dv Y \mkv \dv U_{\mc K\setminus k }$.

For the third term in \eqref{eq:FirstOrder},
\begin{align}
\dv K_{\dv u_{\mc K\setminus k}^t|\dv u_k^t}& \dv \Sigma_{\dv u_{\mc K\setminus k}^t|\dv u_k^t}^{-1}\boldsymbol\mu_{\dv u_{\mc K\setminus k}^t|\dv x_k}
%\nonumber\\
%&
=
\dv \Sigma_{\dv u_k^t|\dv u_{\mc K\setminus k}^t}^{-1}\dv  A_{k}^{t}(\dv I - \dv \Sigma_{\dv x_k|\dv u_{\mc K\setminus k}^t}\dv \Sigma_{\dv x_k}^{-1}) \dv x_k.
\end{align}

Equation \eqref{eq:AUpdate} follows by noting that $\boldsymbol\mu_{\dv u_{k}^{t+1}|\dv x_k} = \dv A_{k}^{t+1}\dv x_k$, and that from \eqref{eq:FirstOrder} $\dv A_{k}^{t+1}$ can be identified as  in \eqref{eq:AUpdate}. 

Finally, note that due to \eqref{eq:testChan}, $\dv\Sigma_{\dv u_k^t|\dv y}$ and $\dv\Sigma_{\dv u_k^t|\dv u_{\mc K\setminus k}^t}$ are given as in \eqref{eq:Cov_ux} and \eqref{eq:Cov_uus}, where $\dv\Sigma_{\dv x_k|\dv y}=\dv\Sigma_{k}$ and $\dv\Sigma_{\dv x_k|\dv u_{\mc K\setminus k}^t}$ follows from its definition.
This completes the proof.\qed

\ifCLASSOPTIONcaptionsoff
  \newpage
\fi



%\bibliographystyle{IEEEtran}
%\begin{thebibliography}{1}
%\end{thebibliography}

\bibliographystyle{IEEEtran}
\bibliography{ref}

\end{document}